%% file: main.tex
\documentclass{article}
\usepackage{def-hao}
\usepackage{fullpage}
\usepackage[utf8]{inputenc}
\usepackage[T1]{fontenc}
\usepackage{booktabs}
\usepackage{makecell}
\usepackage{enumitem}
\usepackage{titlesec}
\usepackage[section]{placeins}
\usepackage[square]{natbib}
\setcitestyle{authoryear,open={(},close={)}}
\hypersetup{
    hypertexnames=false,
    hidelinks,
    pdftitle={Optimal and Efficient Contextual Combinatorial Semi-bandits with General Function Approximation},
    pdfauthor={Hao Qin and Chicheng Zhang}
}
\title{Optimal and Efficient Contextual Combinatorial Semi-bandits with General Function Approximation}

\author{
\begin{tabular}{c}
{\Large Hao Qin~~~~~~~~~~Chicheng Zhang}\\
The University of Arizona\\
{\small\texttt{\{hqin,chichengz\}@arizona.edu}}
\end{tabular}
}
\newcommand{\mypara}[1]{\noindent\textbf{#1}}
\date{}

\begin{document}
\input{Contents/0.acronyms}

\maketitle

\begin{abstract}
We study the contextual combinatorial semi-bandit (CCSB) problem, where at each round, the learner observes a context, selects a combinatorial action consisting of a subset of basic arms, and receives the reward of each selected arm; its goal is to maximize its cumulative reward over time.
While many previous works focus on CCSB problems with linear reward function approximation, the design of efficient algorithms for the CCSB problem with general reward function approximation is limited.
In this work, we study the realizable setting where the learner has access to a reward function class $\mathcal{F}$ that contains the true reward function, and propose the SquareCB.Comb algorithm.
We prove that SquareCB.Comb achieves a minimax optimal regret bound of $O(\sqrt{m A T \log |\mathcal{F}|})$, where $A$ is the number of arms, $m$ is the maximum number of arms in a combinatorial action, and $T$ is the time horizon.
In addition, SquareCB.Comb is computationally efficient when we have an online regression oracle that can efficiently predict rewards of all arms given contexts, and
the combinatorial action set admits either efficient linear optimization or efficient sampling given any target marginal probability on arms.
We empirically verify the effectiveness of our algorithm in webpage recommendation and contextual shortest path problems.

\end{abstract}

\input{Contents/1.introduction}

\input{Contents/2.related-work}

\input{Contents/3.problem-setup}

\input{Contents/4.algorithm-design}

\input{Contents/5.experiments}

\mypara{Acknowledgments.} We thank the anonymous reviewers for their constructive feedback, and \citet{krishnamurthy2016contextual} for making their source codes available, which we build upon.
We also thank National Science Foundation grants IIS-2440266 (CAREER) and CCF-2327013 for funding support.

\clearpage
\bibliography{references,more_references}

\clearpage
\appendix

\crefname{section}{appendix}{appendices}
\Crefname{section}{Appendix}{Appendices}

\input{Contents/9.0.appendix}

\end{document}

%% file: Contents/0.acronyms.tex
\newcommand{\sgreedy}[2]{\hat{\pi}_{#1}(#2)}
\newcommand{\coverage}{\ensuremath{\mathsf{Coverage}}}
\newcommand{\conv}{\ensuremath{\mathsf{Conv}}}
\newcommand{\est}{\ensuremath{\mathsf{Est}}}
\newcommand{\dec}{\ensuremath{\mathsf{dec}}}
\newcommand{\decsq}{\ensuremath{\mathsf{dec}^{\mathsf{sq}}}}
\newcommand{\doec}{\ensuremath{\mathsf{doec}}}
\newcommand{\SEC}{\ensuremath{\mathsf{SEC}}}
\newcommand{\WSEC}{\ensuremath{\mathsf{W}\text{-}\mathsf{SEC}}}
\newcommand{\Reg}{\mathrm{Reg}}
\newcommand{\Off}{\mathrm{Off}}
\newcommand{\On}{\mathrm{On}}
\newcommand{\sqcomb}{\textnormal{\textsc{SquareCB.Comb}}\xspace}
\newcommand{\oetdcomb}{\textnormal{\textsc{OE2D.Comb}}}
\newcommand{\sqlin}{\textnormal{\textsc{SquareCB.Lin}}}
\newcommand{\falconlin}{\textnormal{\textsc{FALCON.Lin}}}
\newcommand{\err}{\ensuremath{\mathsf{Err}}}
\newcommand{\sampleoracle}{\Ocal_{\textnormal{\textsc{sample}}}}
\newcommand{\batchoracle}{\Ocal_{\textnormal{\textsc{batch}}}}
\newcommand{\offoracle}{\Ocal_{\textnormal{\textsc{batch.off}}}}
\newcommand{\Alg}{\mathrm{Alg}}
\newcommand{\Unif}{\mathrm{Unif}}
\newcommand{\rel}{\mathrm{rel}}
\newcommand{\smax}{s_{\max}}

%% file: Contents/1.introduction.tex
\section{Introduction}

The contextual combinatorial semi-bandits (CCSB) problem is a framework for sequential decision making in which the learner repeatedly selects a subset of base arms (i.e., a combinatorial action) based on observed contextual information, with the goal of maximizing its cumulative reward over time.
For example, \citet{sandberg2025bayesian} formulate energy-efficient vehicle navigation as a CCSB problem, where the learner repeatedly selects a source-to-destination path (a subset of road segments) and seeks to minimize the total energy consumption along the chosen path, using contextual side information such as road-segment length and traffic conditions.
As another example, \citet{wang2017efficient} formulate contextual whole-page recommendation as a CCSB problem, where for each incoming user the learner selects a feasible subset of item--position pairs (a combinatorial action), observes click feedback for each displayed item, and seeks to maximize the total number of clicks.
Other applications include resource allocation in wireless networks, such as channel allocation~\citep{gai2012combinatorial} and beam management~\citep{li2025contextual}.

In a CCSB environment, we are given a context space $\Xcal$ (which can be infinite) and a finite arm set $\Acal$ of size $A$.
The combinatorial action space $\Scal \subseteq \cbr{0,1}^A$ is a feasible family of subsets of $\Acal$ determined by the problem structure (e.g., the $m$-set setting~\citep{kale2010non} or knapsack constraints~\citep{cacchiani2022knapsack}).
At each round $t$ of interaction, the environment shows a context $x_t$ to the learner, together with a vector $r_t \in [0, 1]^\Acal$ representing the reward of each arm, hidden from the learner.
The learner selects a combinatorial action $s_t \in \Scal$ based on $x_t$, observes the semi-bandit feedback $\cbr{r_t(a) : s_t(a) = 1}$, and receives reward $\inner{r_t}{s_t} = \sum_{a: s_t(a)=1} r_t(a)$. The goal is to minimize the regret over $T$ rounds, defined as the gap between the cumulative reward of the optimal context-dependent action in hindsight and the algorithm's expected cumulative reward.

Prior works on contextual combinatorial semi-bandits with function approximation focus on the linear-reward setting, in which the expected reward of each arm is a linear function of the context~\citep{qin2014contextual,li2025efficient,takemura2021near}.
In many real-world applications, however, the mapping from context and arm to reward is not strictly linear. For instance, user clicks in online recommendation systems often exhibit highly non-linear dependence on contextual features.
In this work, we consider a general function approximation setting for CCSB, where the mapping from the context to the
expected reward of the arms is assumed to lie in some known function class $\Fcal$.

\paragraph{Our contributions.} We propose an algorithm, \sqcomb, for the CCSB problem with general reward function approximation.
Our algorithm builds on the Estimation-to-Decisions principle for online decision making~\citep{foster2021statistical}, which lets the algorithm balance exploration and exploitation and yields a regret bound of $\iupboundlog{\sqrt{m A T\log\abs{\Fcal}}}$, where $m$ is the maximum number of arms in a combinatorial action, $A$ is the number of arms, $T$ is the time horizon, and $\Fcal$ is the reward function class.
This bound matches the best known regret rate for the linear-reward setting and extends it to general reward function approximation.
Given the standard assumption of accessing an efficient oracle for online regression~\citep{foster2020beyond}, \sqcomb\ can be implemented efficiently under two different settings:
first, the convex hull of the combinatorial action set can be represented succinctly, and there exists an efficient procedure that samples a combinatorial arm given any target marginal probability on arms (Mode 1); second, when the combinatorial action set admits an efficient linear optimization oracle, in line with prior works~\citep{neu2016importance} (Mode 2).\footnote{We remark that this is a minimal assumption for the design of efficient CCSB algorithms -- without such an assumption, even the more basic noncontextual optimization with known rewards can be intractable, such as the $0$--$1$ knapsack problem~\citep{cacchiani2022knapsack}.} We empirically verify the effectiveness of our algorithm in webpage recommendation and contextual shortest path problems.

\begin{table}[t]
    \centering
    \caption{
    A comparison between our work and prior works on contextual combinatorial semi-bandits with general function approximation.
    $\iupboundlog{\cdot}$ hides logarithmic factors in $T$, $A$, and $\delta$.
    $d$ is the dimension of the feature in the linear reward setting, and $\abs{\Pi}$ is the size of the policy class in the general reward setting.
    SquareCB.lin is a contextual bandit algorithm adapted to the semi-bandit setting by aggregating per-arm feedback into a slate-level reward.
    In the policy-class setting, regret is measured against the best policy in $\Pi$ in hindsight; in the value-function-approximation setting, the learner is assumed to be realizable with respect to $\Fcal$ and competes against the optimal context-dependent action in hindsight.
    }
    \label{tab:algorithms}
    \renewcommand{\arraystretch}{1.15}
    \resizebox{\textwidth}{!}{%
    \begin{tabular}{ccccc}
        \toprule
        Algorithm  & \makecell{Problem\\ Setting} & \makecell{Combinatorial\\ Constraint} & \makecell{Approximation\\ Type} & \makecell{Regret\\ Guarantee} \\
        \midrule
        \citet{kale2010non}
        & \makecell{Stochastic context,\\ adversarial reward} & \makecell{$m$-set \\ $m$-ordered slate} & Policy & $\iupboundlog{\sqrt{m A T \log(\abs{\Pi})}}$ \\
        \addlinespace[2pt]
        SquareCB.lin~\citep{foster2020adapting}
        & \makecell{Adversarial context,\\ stochastic reward} & General & Value & $\iupboundlog{m\sqrt{AT \log(\abs{\Fcal})}}$\\
        \midrule
        SquareCB.comb (ours)
        & \makecell{Adversarial context,\\ stochastic reward} & General & Value & $\iupboundlog{\sqrt{m A T \log(\abs{\Fcal})}}$ \\
        \bottomrule
    \end{tabular}%
    }
\end{table}

%% file: Contents/2.related-work.tex
\section{Related Work}
\label{sec:related-work}

\paragraph{CCSB with linear function approximation.}
Motivated by personalized recommendation, \citet{qin2014contextual} first introduced the contextual combinatorial semi-bandits (CCSB) problem under the assumption that the expected reward of each arm is linear in the context with a shared $d$-dimensional coefficient, while the reward of a combinatorial action may be a general function of the per-arm rewards.
They proposed C$^2$UCB, a UCB-based algorithm with an $\iupbound{\sqrt{(d+m) d m T}}$ regret guarantee.
\citet{li2016contextual} studied a variant of the CCSB problem with a cascade feedback structure and proposed an algorithm with $\iupbound{\sqrt{d m T}}$ regret under a (strong) feedback coverage assumption.
\citet{takemura2021near} gave a refined analysis of C$^2$UCB and obtained a minimax optimal regret guarantee of $\iupbound{d\sqrt{m T}}$. In comparison, our algorithm, instantiated with a discretization of the linear reward function class, attains regret $\iupboundlog{\sqrt{m A d T}}$ in this setting.\footnote{It may appear that their lower bound contradicts ours when $d$ is large; however, a close examination of their lower-bound instances shows that they take $A = m 2^d$.}
\citet{zierahn2023nonstochastic} studied the adversarial linear reward CCSB problem, where each arm's linear reward coefficient is chosen adversarially, and proposed a follow-the-regularized-leader (FTRL) algorithm with $\iupbound{\sqrt{T}}$ regret.
\citet{liu2023contextual} generalizes the results of \citet{li2016contextual,takemura2021near} to linear contextual combinatorial bandits with probabilistically triggered arms; when instantiated in the setting of \citet{li2016contextual}, their result replaces the minimum triggering probability assumption with a natural assumption on the reward function.
\citet{li2025efficient} gave the first best-of-both-worlds algorithm for linear CCSB, attaining $\iupbound{\sqrt{T}}$ regret in the adversarial setting and $\iupbound{\log T}$ regret in the stochastic setting by combining FTRL with a negative Shannon entropy regularizer.

\paragraph{CCSB with nonlinear function approximation.}  \citet{levy2023efficient} studied the contextual MDP (CMDP) problem with adversarial context and general function approximation, which can be used to solve the contextual shortest-path problem as a special case.
On a graph with $N$ nodes and $M$ edges, a shortest-path instance can be cast as an episodic CMDP with episode length $N$; under this reduction, their algorithm attains regret $\iupboundlog{\sqrt{N^{7} T \log\abs{\Fcal}}}$,
which is worse than our bound $\iupboundlog{\sqrt{N M T \log\abs{\Fcal}}}$ in this setting (see \Cref{sec:regret-sqcomb} for details).
\citet{hwang2023combinatorial} studied neural network function approximation with a Lipschitz continuity assumption and proposed algorithms based on the Thompson sampling and UCB principles, achieving $\iupboundlog{\sqrt{m \tilde{d} T}}$, where $\tilde{d}$ is the effective dimension of the neural tangent kernel associated with the reward function class.

Another line of work approaches CCSB via policy search: the learner is given a policy class $\Pi$ that maps contexts to combinatorial actions, and seeks to minimize regret against the best policy in $\Pi$ in hindsight.
The early work of~\citet{kale2010non} proposed algorithms
based on exponential weighting over the policy class, achieving $\iupboundlog{\sqrt{m A T \log(\abs{\Pi})}}$ regret in the unordered and ordered slate settings with $n$ items (where $A = n$ and $mn$ respectively). Their algorithm is computationally inefficient as it requires enumerating all policies in the policy class.
Subsequently, \citet{krishnamurthy2016contextual} studied CCSB
in the slate setting with a linearly weighted reward structure and proposed the VCEE algorithm, which improves computational efficiency by extending the soft policy elimination technique of \citet{agarwal2014taming} based on a cost-sensitive classification oracle over the policy class rather than enumerating all policies. Their regret guarantee ranges between $\iupboundlog{\sqrt{mAT\log\abs{\Pi}}}$ and $\iupboundlog{m\sqrt{AT\log\abs{\Pi}}}$, depending on the reward structure.
\citet{erez2025contextual} studied CCSB in both the PAC guarantee and the regret minimization settings, and adapted the technique of \citet{agarwal2014taming} to construct an importance-weighted sampler over the policy class, achieving an $\iupboundlog{\abs{\Pi} + \sqrt{m s T \log \abs{\Pi}}}$ regret guarantee under the sparsity assumption $\|r_t\|_1 \leq s$.

\paragraph{Applying contextual bandit algorithms to the CCSB problem.} A natural approach to the CCSB problem is to reduce it to a contextual bandit problem with reward being linear in action per-context~\footnote{This is different from the aforementioned ``reward linear in context''setting. }
Specifically, upon seeing a combinatorial action $s_t \in \cbr{0,1}^A$, the reduction aggregates the semi-bandit feedback $(r_t(a))_{s_t(a)=1}$ to
a bandit feedback
$\inner{r_t}{s_t}$, where $r_t \in \RR^A$ is the reward vector of the $t$-th example.
Under this canonical reduction, off-the-shelf regression-based contextual bandit algorithms apply directly, including SquareCB.Lin~\citep{foster2020adapting}, the linear FALCON algorithm of \citet{xu2020upper}, and the SpannerIGW algorithm of \citet{zhu2022contextualb};
all of these algorithms operate in the linear reward structure per-context setting and generalize the inverse-gap-weighting (IGW) algorithm of \citet{abe1999associative}, which learns an estimate of the reward function and uses it to sample actions.
Applying SquareCB.Lin under this reduction yields a regret of $\iupboundlog{m\sqrt{A T \log\abs{\Fcal}}}$ in the $m$-set setting, where $\Fcal$ denotes the underlying reward function class. The extra factor of $\sqrt{m}$ originates from the (data-inefficient) update of the regression oracle based on aggregated rewards; see Appendix~\ref{sec:cb-ccsb} for details.

We compare our work with the most related prior works in Table~\ref{tab:algorithms}. Due to space limits, additional related works are discussed in Appendix~\ref{sec:addl-related-work}.

%% file: Contents/3.problem-setup.tex
\section{Preliminaries}
\label{sec:prelim}
\paragraph{Basic notations.}
For any two vectors $x, y \in \RR^A$, we write $\iinner{x}{y} \coloneqq \sum_{a \in \Acal} x(a) y(a)$ for their inner product. For any set $\Scal \subseteq \RR^A$, we write $\conv(\Scal)$ for its convex hull, and $\opnorm{f}{p} = \sqrt{\sum_{a \in \Acal} p(a)f(a)^2}$ for the $L^2(p)$ norm of a vector $f \in \RR^A$.
$\Delta(\Scal)$ denotes the set of distributions over $\Scal$.
$\iupbound{\cdot}$ and $\ilowbound{\cdot}$ denote upper and lower bounds that hide absolute constants; their tilde variants $\iupboundlog{\cdot}$ and $\ilowboundlog{\cdot}$ additionally hide polylogarithmic factors in $A, m, T, \log\abs{\Fcal}$.
We write $\lesssim$ for inequalities that hide absolute constants.
Let $\bar{\Scal} \coloneqq \conv(\Scal)$ be the convex hull of the combinatorial action set. Given a distribution $p \in \Delta(\Scal)$, its \textit{marginal probability vector} $\bar{p} \in \bar{\Scal}$ records the marginal arm-selection probabilities $\bar{p}(a) \coloneqq \EE_{s \sim p}\sbr{s(a)}$ for all $a \in \Acal$.
A \textit{policy} is a (measurable) map $\pi: \Xcal \to \Delta(\Scal)$. Given a context $x$, we write $\bar{\pi}(x) \in \bar{\Scal}$ for the marginal probability vector of distribution $\pi(x) \in \Delta(\Scal)$.
Note that $\bar{\pi}(x)$ is $A$-dimensional, and its $a$-th coordinate $\bar{\pi}(x)(a)$ is the marginal probability of selecting arm $a$ under $\pi(x)$.

\paragraph{Contextual combinatorial semi-bandits.}
Let $\Xcal$ be a context space and $\Acal$ a finite set of arms with size $A$. The set of available combinatorial actions is denoted as $\Scal \subseteq \cbr{0,1}^{A}$, where for any combinatorial action $s \in \Scal$, $s(a) = 1$ encodes that arm $a$ is selected by $s$. The agent is given a class $\Fcal$ of mean reward functions $f: \Xcal \times \Acal \to [0,1]$, and we adopt two standard assumptions.
\begin{assum}[Realizability]
    \label{assum:realizable}
    The ground-truth reward function $f^\star$ lies in $\Fcal$.
\end{assum}
\begin{assum}[Bounded combinatorial-action size]
    \label{assum:action-size}
    There exists a known constant $m \in \NN$ such that $\|s\|_1 \leq m$ for every $s \in \Scal$.
\end{assum}

Realizability is a standard assumption in the contextual bandit literature and is necessary for sublinear regret with general function classes~\citep{foster2020beyond,foster2020adapting}.
\citet{lattimore2020learning} show that, without realizability, linear regret of order $\varepsilon T$ is unavoidable for linear contextual bandits with misspecification level $\varepsilon$.
Although it is always true that the maximum combinatorial action size $m \leq A$,
analyzing the (fine-grained) performance of algorithms in terms of $m$ is common in the combinatorial bandit literature, and is considered in many applications including unordered and ordered slate selection~\citep{kale2010non, kveton2015tight}, maximum-weight matching~\citep{gai2012combinatorial}.

\paragraph{Interaction protocol.}
At each round $t = 1, \ldots, T$:
\begin{enumerate}
    \item the environment draws a context $x_t$ and a reward vector $r_t \in [0,1]^A$, with $x_t$ revealed to the agent.
    We assume that $\EE\sbr{ r_t(a) \mid x_t = x } = f^\star(x, a)$;
    \item the agent selects a combinatorial action $s_t \in \Scal$;
    \item the agent observes semi-bandit feedback $o_t = (r_t(a))_{a: s_t(a) = 1}$.

\end{enumerate}
The expected reward of combinatorial action $s$ on context $x$ takes a linear form $\iinner{s}{f(x, \cdot)}$, where $f(x, \cdot) \in \RR^{A}$ stacks the per-action rewards. We denote the optimal combinatorial action at context $x$ by $s^\star(x) \coloneqq \argmax_{s \in \Scal}\iinner{s}{f^\star(x, \cdot)}$.
After $t$ rounds, the agent has collected a dataset $\Dcal_t = \cbr{(x_i, s_i, o_i)}_{i=1}^t$.

\paragraph{Running examples.} We will demonstrate the utility of our problem setting and algorithms on the following examples throughout this paper:
\begin{itemize}
\item \textit{Unordered $m$-set recommendations with $n$ items~\citep{kale2010non}.} Here we have $A = n$ available items for recommendation, $\Scal=\icbr{s \in \icbr{0,1}^n: \| s \|_1 = m}$ contains all subsets of size $m$ that can be recommended at each round. Each item is regarded equally, and we observe the user's click-through rates on the items recommended.
\item \textit{Ordered $m$-slate recommendations with $n$ items~\citep{kale2010non}.} Here, the user's click-through rate depends on both the item recommended and its position in the ranked list.
Here, our combinatorial action set is the set of permutations:
$\Scal = \icbr{ M \in \icbr{0,1}^{m \times n}: \sum_{j=1}^n M_{i,j} = 1, \forall i \text{ and } \sum_{i=1}^m M_{i,j} \leq 1, \forall j }$, where $M_{i,j} = 1$ indicates that item $j$ is recommended at position $i$. Therefore, $A = m \cdot n$.
\item \textit{Contextual shortest path on a directed acyclic graph (DAG).} Motivated by energy-efficient navigation~\citep{sandberg2025bayesian}, we consider a DAG with $N$ vertices $V$ and $M$ edges $E$, together with a designated source $s$ and terminal $t$; $\Scal = \icbr{ (x_e)_{e \in E}: \sum_{e \in N_+(v)} x_e - \sum_{e \in N_-(v)} x_e = \mathbb{I}(v = t) - \mathbb{I}(v = s),\ x_e \in \cbr{0,1} }$ is the set of all $s$-$t$ paths, where $N_+(v)$ and $N_-(v)$ are the sets of edges entering and leaving $v$, respectively. Semi-bandit feedback means that the learner only sees the costs on the road segments it traverses. In this example, $A = M$ and we may take $m \le N$.
\end{itemize}

The agent's goal is to minimize its cumulative (pseudo-)regret:
\begin{align}
    \textstyle \Regret(T) \coloneqq \sum_{t=1}^T \inner{s^\star(x_t)}{f^\star(x_t, \cdot)} - \sum_{t=1}^T \inner{ s_t }{f^\star(x_t, \cdot)}.
        \label{eqn:regret}
\end{align}

\paragraph{Batch-mode online regression oracle.}

Online regression oracles are commonly used in the design and analysis of contextual bandit algorithms with general function approximation~\citep{foster2020beyond,zhu2022contextuala}.
Unlike bandit feedback, which reveals one reward per round, semi-bandit feedback reveals up to $m$ rewards, one for each selected arm. Since all predictions for a round are made before its feedback is observed, we adopt a batch-mode formulation of the online square-loss regression oracle $\batchoracle$, defined as follows. Let $\Gcal$ be a finite function class. At each round $t$, the regression oracle first outputs a predictor $\hat g_t$ based on past batches, and then receives a batch $\cbr{(z_{t,b}, y_{t,b})}_{b=1}^{B_t}$ of $B_t \leq B$ labeled examples.
The regression oracle's performance is measured by the square-loss regret compared to the best predictor in $\Gcal$ on the sequence of batches:
\begin{align*}
    \textstyle
    \Reg_{\mathrm{batch}}(T)
    =
    \sum_{t=1}^T \sum_{b=1}^{B_t}\del{\hat g_t(z_{t,b}) - y_{t,b}}^2
    -
    \inf_{g \in \Gcal} \sum_{t=1}^T \sum_{b=1}^{B_t}\del{g(z_{t,b}) - y_{t,b}}^2.
\end{align*}

We make the following assumption on the batch-mode online regression oracle $\batchoracle$:
\begin{assum}[Batch-mode online regression oracle]
    \label{assum:online-reg-semi}
    Suppose the reward function $g^\star$ belongs to a finite function class $\Gcal$. The batch-mode online regression oracle $\batchoracle$ with batch size $B$ guarantees that, for any sequence of inputs and outputs $\cbr{(z_{t,b}, y_{t,b})}_{t=1}^T$, where $y_{t, b} \in [0, 1]$ for all $t, b$, the predictors $\hat g_t$'s output by $\batchoracle$ satisfy $\Reg_{\mathrm{batch}}(T) \lesssim B \log |\Gcal|.$
\end{assum}

Indeed, \citet{mesterharm2005line} gives a reduction from the batch/delayed-feedback setting to the standard online learning setting that incurs an additional multiplicative factor of $B$ relative to the standard regret bound. Furthermore, when $\Gcal$ is finite, the Exponentially Weighted Average (EWA) algorithm attains regret $\iupbound{\log\abs{\Gcal}}$ in the standard online learning setting~\citep{cesa2006prediction,rakhlin2014statistical}, which translates into regret $\iupbound{B\,\log\abs{\Gcal}}$ in the batch-mode setting.

%% file: Contents/4.algorithm-design.tex
\section{Proposed Algorithm}
\label{sec:algorithm}

\label{sec:reduction}

We propose an algorithm for the contextual combinatorial semi-bandit setting, called \sqcomb\ (\Cref{alg:sqcomb}), which achieves $\iupboundlog{\sqrt{m A T \log\abs{\Fcal}}}$ regret under the realizability assumption. The algorithm is inspired by the SquareCB algorithm for contextual bandits~\citep{foster2020beyond}, whose action-selection rule is the solution of a log-barrier regularized optimization problem~\citep{foster2020adapting} every round.
At each round, the learner obtains an estimate of the reward function by calling a batch-mode online regression oracle $\batchoracle$ trained on past observations of the form (context, arms, rewards) (line~\ref{line:sqcomb-regression}), and observes the context $x_t$ (line~\ref{line:sqcomb-context}).
A sampling oracle $\sampleoracle$ takes the marginal probability vector $\bar{p}$ and returns a combinatorial action $s$ sampled from a distribution whose the expectation of the incidence vector is equal to $\bar{p}$.
If such an $\sampleoracle$ is available, the learner uses Mode 1: it computes a marginal probability vector $\bar{p}_t$ over arms by solving a log-barrier regularized optimization problem (line~\ref{line:sqcomb-logbarrier}), and invokes $\sampleoracle$ to sample a combinatorial action $s_t$ whose marginals match $\bar{p}_t$ (line~\ref{line:sqcomb-sampling}). Alternatively, if a linear optimization oracle over $\Scal$ is available, the learner uses Mode 2: it calls \Cref{alg:sample-soln} with this oracle to find a distribution over combinatorial actions directly, and samples $s_t$ from this distribution (lines~\ref{line:sqcomb-linear-oracle}--\ref{line:sqcomb-direct-sampling}). Finally, the learner observes the semi-bandit feedback $o_t$ (line~\ref{line:sqcomb-feedback}) and forwards the new observation to $\batchoracle$ for use in the next round.

\paragraph{Mode 1: using a problem-dependent sampling oracle.} In Mode 1, the learner solves the following $A$-dimensional convex optimization problem (Eq.~\eqref{eqn:log-barrier}) to obtain a marginal probability vector $\bar{p}_t$.
The first term in the objective, $\iinner{\bar{p}}{\hat f_t(x_t,\cdot)}$, measures the ``greediness'' of the marginal probability vector $\bar{p}$, i.e., how optimal $\bar{p}$ is with respect to the current reward function estimate $\hat{f}_t$ and context $x_t$;
the second term $\frac{1}{\gamma}\sum_{a \in \Acal}\log\bar{p}(a)$ is a log-barrier regularizer that encourages $\bar{p}$ to spread its support, thereby preventing concentration on a small set of arms and promoting exploration. By balancing these two terms, the learner obtains a marginal probability vector that exploits the current reward estimate while ensuring sufficient exploration to control the estimation error.

\begin{algorithm}[t]
    \small
    \begin{algorithmic}[1]
    \STATE \textbf{Input:} Semi-bandit tuple $(\Xcal, \Scal, \Acal)$, online regression oracle $\batchoracle$, reward function class $\Fcal$, exploration parameter $\gamma$,
    mode $\in\{1,2\}$; a problem-dependent sampling oracle $\sampleoracle$ for Mode 1, or a linear optimization oracle over $\Scal$ for Mode 2. \label{line:sqcomb-input}
    \FOR{round $t = 1, 2, \ldots$}
        \STATE Receive an updated estimate $\hat f_t$ from $\batchoracle$ based on the history $\cbr{(x_\tau, s_\tau, o_\tau)}_{\tau < t}$. \label{line:sqcomb-regression}
        \STATE Observe context $x_t$. \label{line:sqcomb-context}
        \IF{Mode 1 is selected}
        \STATE Compute a marginal probability vector $\bar{p}_t \in \bar\Scal$ by solving
        \begin{align}
        \bar{p}_t
        =
        \argmax_{\bar{p} \in \bar\Scal}
        \inner{\bar{p}}{\hat f_t(x_t,\cdot)}
        +
        \frac{1}{\gamma}\sum_{a \in \Acal}\log\bar{p}(a).
        \label{eqn:log-barrier}
        \end{align}

        \label{line:sqcomb-logbarrier}
        \STATE Sample a combinatorial action $s_t$ such that $\EE\sbr{s_t} = \bar{p}_t$ by calling $\sampleoracle(\bar{p}_t, \Scal)$.
        \label{line:sqcomb-sampling}
        \ELSE
        \STATE \COMMENT{Mode 2 is selected:}  Call \Cref{alg:sample-soln} with $\hat g_t=\hat f_t(x_t,\cdot)$, $\gamma$, $\varepsilon=\gamma^{-2}$, and the linear optimization oracle to obtain $p_t=\sum_{i=1}^{n_t}\alpha_t^i\delta_{s_t^i}$, where $\alpha_t\in\Delta^{n_t-1}$ and $s_t^i\in\Scal$.
        Denote by the associated marginal probability vector $\bar{p}_t = \EE_{s \sim p_t}\sbr{s}$.
        \label{line:sqcomb-linear-oracle}
        \STATE Sample $I_t\sim\mathrm{Categorical}(\alpha_t)$ and play $s_t=s_t^{I_t}$.
        \label{line:sqcomb-direct-sampling}
        \ENDIF
        \STATE Observe semi-bandit feedback $o_t = \cbr{r_t(a): s_t(a) = 1}$.
        \label{line:sqcomb-feedback}

    \ENDFOR
    \end{algorithmic}
    \caption{\sqcomb}
    \label{alg:sqcomb}
\end{algorithm}

Next, the learner needs to sample a combinatorial action $s_t$ from a distribution $p_t \in \Delta(\Scal)$ whose marginal expectation matches the marginal probability vector $\bar{p}_t$ computed in Line~\ref{line:sqcomb-logbarrier}.
This step is nontrivial because the marginal probability vector $\bar{p}_t$ is defined in the convex hull of the combinatorial action space $\bar\Scal$, whose dimension is $A$, while the actual sampling must be performed in $\Delta(\Scal)$, a space of dimension $\abs{\Scal}$ that can be exponentially larger than $A$.
For many structured combinatorial action sets, however, efficient sampling algorithms are readily available. For instance:
\begin{itemize}
\item \textit{Unordered $m$-set recommendations with $n$ items.}
Here, after computing $\bar{p}_t$, we can use dependent rounding~\citep{gandhi2006dependent} to sample $s_t$ with running time $O(n)$.

\item \textit{Ordered $m$-slate recommendations with $n$ items.}
\citet{kale2010non} give an efficient sampler (which in turn uses the efficient algorithm of~\citet{helmbold2009learning}) that first decomposes any $\bar{p} \in \bar{\Scal}$ to a convex combination of $n^2$ matrices, and samples from it.

\item \textit{Contextual shortest path on DAG.} Any $\bar{p} \in \bar{\Scal}$ is a unit $s$-$t$ flow, and can be represented as a convex combination of $s$-$t$ paths via standard flow decomposition procedures~\citep{ahuja1988network}.

\end{itemize}

\paragraph{Mode 2: sampling using a linear optimization oracle.}
When $\Scal$ admits an efficient linear optimization oracle (i.e., $\argmin_{s \in \Scal} \inner{\ell}{s}$ is tractable for any $\ell$, the same linear optimization oracle assumption is used by \citet{neu2016importance} for adversarial combinatorial semi-bandits), as in shortest path and maximum weight matching, Mode 2 jointly implements marginal probability vector computation and sampling: it calls \Cref{alg:sample-soln} (deferred to Appendix~\ref{sec:sample-linear-opt-oracle}), a Frank-Wolfe style algorithm that produces a distribution $p_t=\sum_{i=1}^{n_t}\alpha_t^i\delta_{s_t^i}$ with marginal probability vector $\bar p_t=\sum_{i=1}^{n_t}\alpha_t^i s_t^i$ with $s_t^i \in \Scal$ and $\alpha_t \in \Delta^{n_t-1}$, so sampling $I_t \sim \mathrm{Categorical}(\alpha_t)$ and returning $s_t^{I_t}$ yields $s_t$ with marginals $\bar p_t$.
At each iteration, \Cref{alg:sample-soln} calls the linear optimization oracle to find a combinatorial action that maximizes predicted reward plus an inverse-marginal bonus; thus it favors high predicted rewards (exploitation) and assigns some mass for exploration. The final distribution is a convex combination of the iterates, and satisfies a guarantee (IF-condition in Step 7 is false) that balances exploration and exploitation.
Mode 2 avoids a case-by-case design of specialized sampling procedures, and it is also computationally efficient given efficient regression and linear optimization oracles; see Appendix~\ref{sec:semibandit-computation} for its end-to-end cost.

\subsection{Regret Analysis of \sqcomb}
\label{sec:regret-sqcomb}

We now present the main regret guarantee of \sqcomb. In \Cref{thm:regret-sqcb-comb-regsq}, we give a general regret bound in terms of the online regression oracle's square-loss regret $\Reg_{\mathrm{batch}}(T)$, which, combined with the guarantee of \Cref{assum:online-reg-semi}, yields the explicit bound in \Cref{corol:regret-sqcb-comb}.
\begin{theorem}[Total regret bound]
    \label{thm:regret-sqcb-comb-regsq}
    Suppose \Cref{assum:realizable,assum:action-size} hold and \sqcomb\ uses a batch online regression oracle with square-loss regret bound $\Reg_{\mathrm{batch}}(T)\geq1$ and predictions in $[0,1]$, and a constant exploration parameter $\gamma = \sqrt{A T / \Reg_{\mathrm{batch}}(T)}$.
    Then the expected regret of \sqcomb\ is bounded as
    \[
        \EE\sbr{\Regret(T)} \leq \upboundlog{\sqrt{A\,T\,\Reg_{\mathrm{batch}}(T)}}.
    \]
\end{theorem}

Combining \Cref{thm:regret-sqcb-comb-regsq} with the guarantee of the online regression oracle in the finite-class semi-bandit setting gives the following regret bound:
\begin{corollary}
    \label{corol:regret-sqcb-comb}
    Under \Cref{assum:realizable,assum:action-size,assum:online-reg-semi} and $\gamma = \sqrt{\frac{A T}{m \log \abs{\Fcal}}}$, the expected regret of \sqcomb\ is upper bounded by
    \[
    \EE\sbr{\Regret(T)} \leq \upboundlog{\sqrt{m A T \log |\Fcal|}}.
    \]
\end{corollary}

The bound is sublinear in $T$ and exhibits the same $\sqrt{mAT}$ dependence, up to logarithmic factors, as the standard regret rate for the stochastic combinatorial semi-bandit problem under linear rewards~\citep{kveton2015tight}, while permitting the mean reward function to belong to a general class $\Fcal$.
When $m=1$, the contextual combinatorial semi-bandit setting reduces to the standard contextual bandit setting, and \Cref{thm:regret-sqcb-comb-regsq} recovers the $\sqrt{A T \log\abs{\Fcal}}$ regret of the IGW algorithm for contextual bandits~\citep{foster2020beyond} under realizability.

Specializing \Cref{corol:regret-sqcb-comb} to the three running examples gives:
\begin{itemize}
    \item \textit{Unordered $m$-set with $n$ items.} Each arm corresponds to one of the $n$ items, so $A = n$, and \Cref{corol:regret-sqcb-comb} yields regret $\iupboundlog{\sqrt{m n T \log\abs{\Fcal}}}$. This matches the $\iupboundlog{\sqrt{m n T \log\abs{\Pi}}}$ regret of \citet{kale2010non,krishnamurthy2016contextual}, which compete against the best policy in a class $\Pi$, while \sqcomb\ is computationally more efficient since it relies on a better practicality of regression oracles relative to classification oracles.

    \item \textit{Ordered $m$-slate with $n$ items.} Since $A = m n$, \Cref{corol:regret-sqcb-comb} yields regret $\iupboundlog{m\sqrt{n T \log\abs{\Fcal}}}$, comparable to the bound of \citet{kale2010non} with better computational efficiency.
    \item \textit{Contextual shortest path on a DAG.} Since $A = M$ and $m \le N$, \Cref{corol:regret-sqcb-comb} yields regret $\iupboundlog{\sqrt{N M T \log\abs{\Fcal}}}$. In comparison, applying the contextual MDP algorithm of \citet{levy2023efficient} requires solving an $N$-state, $N$-action MDP with episode length $H=N$, which yields regret at least $\iupboundlog{\sqrt{N^7 T \log\abs{\Fcal}}}$.

\end{itemize}

\paragraph{Extensions: infinite classes, misspecified setting, and offline oracle-efficient algorithm.}
\Cref{thm:regret-sqcb-comb-regsq} treats the regression oracle abstractly through its cumulative square-loss regret $\Reg_{\mathrm{batch}}(T)$, which readily yields regret-efficient CCSB algorithms beyond the basic realizable, finite-class setting by plugging in different online regression oracles.
For an infinite class $\Fcal$, instantiating the oracle with online square-loss regression based on sequential covering or sequential Rademacher complexity~\citep{rakhlin2014online,rakhlin2014statistical} replaces $m\log\abs{\Fcal}$ with the corresponding sequential complexity.
For the $\varepsilon$-misspecified setting, where $\inf_{f \in \Fcal}\sup_{x,a}\abs{f(x,a)-f^\star(x,a)} \le \varepsilon$, combining \sqcomb\ with the misspecification-aware online regression oracle of \citet{foster2020adapting} yields regret $\iupbound{\sqrt{mAT\log\abs{\Fcal}} + \varepsilon m \sqrt{A} T}$; the additive $\varepsilon T$ factor is unavoidable in light of the $\Omega(\varepsilon\sqrt{d}T)$ lower bound of \citet{lattimore2020learning} for the misspecified $d$-dimensional linear bandit problem.
In Appendix~\ref{sec:oe2d-comb}, we present a variant of \sqcomb\ that reduces CCSB to $O(\log T)$ offline regression problems while enjoying similar regret guarantees when the contexts are i.i.d.\ Given the wider availability of offline regression guarantees, this variant can be preferable in practice~\citep{foster2024online}.

\paragraph{Anytime exploration.}
\sqcomb\ can be run without knowing $T$ by using $\gamma_t=\sqrt{At/(m\log|\Fcal|)}$ in Eq.~\eqref{eqn:log-barrier} at round $t$. This retains the expected-regret bound $\EE\sbr{\Regret(T)}\leq\Ocal(\sqrt{mAT\log|\Fcal|})$ for every $T\geq1$; see Appendix~\ref{sec:sqcomb-anytime} for the proof.

\subsection{Proof sketch of \texorpdfstring{\Cref{thm:regret-sqcb-comb-regsq}}{Theorem \ref*{thm:regret-sqcb-comb-regsq}}}
\label{sec:proof-sketch}

In this section, we provide a proof sketch of \Cref{thm:regret-sqcb-comb-regsq}.
We aim to adopt an analysis similar to SquareCB~\citep{foster2020beyond} that first bounds the instantaneous regret at round $t$ in terms of the expected regression error at the same round, and then conclude the regret bound by summing over all $t$'s.
However, SquareCB's analysis is inherently tied to bandit feedback (the learner receives a single reward observation per round) and thus does not extend to our semi-bandit setting. Furthermore, the general E2D algorithm~\citep{foster2021statistical} (see Appendix~\ref{sec:dmso} for a brief overview of the DMSO framework and the Decision-Estimation Coefficient) is not directly applicable here, since we do not have exact knowledge of the conditional distribution of $\cbr{r_t(a): s_t(a) = 1}$ given the reward model $f$ and $(x_t, s_t)$.
This inspires us to define a new notion of
DEC customized to square loss and combinatorial semi-bandits:
\begin{definition}
For a regressor $\hat{g}: \Acal \to [0,1]$, a function class $\Gcal$, and a combinatorial action set $\Scal$, define their combinatorial square-loss DEC (CS-DEC) as
\begin{equation}
\dec_\gamma(\hat{g}, \Gcal)
=
\min_{p \in \Delta(\Scal)}
\underbrace{ \max_{q \in \Delta(\Scal), g^\star \in \Gcal} \EE_{s \sim q} \sbr{\inner{g^\star}{s}} - \EE_{s \sim p} \sbr{\inner{g^\star}{s}} -
\gamma \EE_{s \sim p} \sbr{ \sum_{a =1}^A s(a) (\hat{g}(a) - g^\star(a))^2 } }_{ =: F(p)}
,
\label{eqn:dec-distn}
\end{equation}
and define $\dec_\gamma(\Gcal) = \max_{\hat{g}} \dec_\gamma(\hat{g}, \Gcal)$.
\end{definition}
The CS-DEC accounts for the semi-bandit feedback
by offsetting the instantaneous regret with the expected square-loss regression error over the entire set of arms associated with the combinatorial action; unlike DEC, it does not rely on exact probabilistic modeling of the reward distribution.
To use CS-DEC for algorithm design, we need to find a distribution $p \in \Delta(\Scal)$ that approximately minimizes $F(\cdot)$. This is in general difficult, since the corresponding optimization problem is $|\Scal|$-dimensional and $|\Scal|$ may be exponential in $m$. Our key observation is that the CS-DEC admits an equivalent representation in the space of marginal probability vectors, which has the much lower dimension $A$, thereby admitting tractable solutions:

\textbf{Observation.} We have the following equivalent characterization:
\begin{equation}
\dec_\gamma(\hat{g}, \Gcal)
=
\min_{\bar{p} \in \bar{\Scal}} \underbrace{ \max_{\bar{q} \in \bar{\Scal}, g^\star \in \Gcal} \inner{g^\star}{\bar{q}} - \inner{g^\star}{\bar{p}} -
\gamma \sum_{a=1}^A \bar{p}(a) (\hat{g}(a) - g^\star(a))^2. }_{ =: \bar{F}(\bar p)}
\label{eqn:dec-part}
\end{equation}
In addition, any $p$ such that $\EE_{s \sim p}\sbr{s} = \bar{p}$ satisfies that $F(p) = \bar{F}(\bar p)$.

This enables the optimization problem~\eqref{eqn:dec-distn} to be approximately solved in two steps: first, compute a $\bar{p}$ that approximately solves~\eqref{eqn:dec-part}; second, convert $\bar{p}$ into a distribution $p \in \Delta(\Scal)$ satisfying $\EE_{s \sim p}[s] = \bar{p}$.

We show that the marginal probability vector $\bar p_t$ found in line~\ref{line:sqcomb-logbarrier} or line~\ref{line:sqcomb-linear-oracle} satisfies $\bar F(\bar p_t)\lesssim A/\gamma$ for the function class $\Fcal_{x_t} := \cbr{f(x_t, \cdot): f \in \Fcal}$.
To this end, we design a surrogate min-max objective that relaxes the CS-DEC, and establish two intermediate results: (i) the marginal probability vector found in line~\ref{line:sqcomb-logbarrier} or line~\ref{line:sqcomb-linear-oracle} certifies a small optimal value of the surrogate objective; and (ii) this certificate implies a small value of the CS-DEC objective.

Specifically, our surrogate objective is:
\begin{align}
    \inf_{\bar{p} \in \bar{\Scal}}
    \max_{\bar{q} \in \bar{\Scal}}
    \inner{\hat f_t(x_t,\cdot)}{\bar q}
    - \inner{\hat f_t(x_t,\cdot)}{\bar p}
    + \frac{1}{\gamma}\coverage(\bar p, \bar q).
        \label{eqn:doec}
\end{align}
where $\coverage(\bar p, \bar q) = \sum_{a \in \Acal} \frac{\bar q(a)^2}{\bar p(a)}$ quantifies the estimation error under the data collected using action distributions with marginal probability vector $\bar p$ when trying to estimate the reward of action distribution with marginal probability vector $\bar q$.\footnote{Strictly speaking, this is only the surrogate objective for Mode 1. For Mode 2, the surrogate objective uses a slightly different coverage: $\overline{\coverage}_{\varepsilon}(\bar p,\bar q)
    \coloneqq \sum_{a \in \Acal}\frac{\bar q(a)}{\bar p(a)+\varepsilon/A}$; the introduction of $\varepsilon$ is only for technical reasons and we choose $\varepsilon$ to be tiny ($O(\gamma^{-2})$), so it is almost identical to the original coverage.
}
For Mode 1, we show in \Cref{lemma:log-barrier-certify-doec} that the $\bar p_t$ computed by the log-barrier optimization certifies a optimal value at most $A/\gamma$ for Eq.~\eqref{eqn:doec}; subsequently, \Cref{lemma:doec-to-dec} in Appendix~\ref{sec:sqcomb-mode1} shows that this certificate implies $\bar F(\bar p_t)\lesssim A/\gamma$. We show similar bounds for Mode 2 in~\Cref{thm:semibandit-linear-oracle} and~\Cref{lemma:semibandit-approximate-dec}, respectively.

Combining the above two lemmas yields the instantaneous regret decomposition:
\begin{align*}
    \max_{s \in \Scal} \inner{s}{f^\star(x_t,\cdot)}
    -
    \inner{\bar{p}_t}{f^\star(x_t,\cdot)}
    \lesssim
    \frac{A}{\gamma}
    +
    \gamma \opnorm{\hat{f}_t(x_t,\cdot)-f^\star(x_t,\cdot)}{\bar{p}_t}^2.
\end{align*}
By taking the instantaneous regret decomposition at each round and summing over $t$ and taking expectation, we have
\begin{align*}
    \EE\sbr{\Regret(T)}
    &\lesssim \frac{AT}{\gamma}
    +\gamma\sum_{t=1}^T\EE\sbr{\opnorm{\hat f_t(x_t,\cdot)-f^\star(x_t,\cdot)}{\bar p_t}^{2}}
    \\
    &\leq \frac{AT}{\gamma}+\gamma\Reg_{\mathrm{batch}}(T).
\end{align*}
Here, the second inequality is due to the realizability and batch-mode online regression oracle's guarantee.
This proves \Cref{thm:regret-sqcb-comb-regsq} by choosing $\gamma=\sqrt{AT/\Reg_{\mathrm{batch}}(T)}$ to give expected regret $\Ocal(\sqrt{AT\Reg_{\mathrm{batch}}(T)})$.
Under \Cref{assum:online-reg-semi}, substituting $\Reg_{\mathrm{batch}}(T)\lesssim m\log|\Fcal|$ gives the result in \Cref{corol:regret-sqcb-comb}.

\section{Lower Bound}
\label{sec:lower-bound-main}

We complement the upper bound of \Cref{corol:regret-sqcb-comb} with a matching lower bound, showing that \sqcomb\ is minimax optimal in $m$, $A$, $T$, and $\log\abs{\Fcal}$ up to logarithmic factors. The lower bound is parameterized by an upper bound $N$ on the size of the reward function class $\Fcal$ used in the construction.

\begin{restatable}[Lower bound for CCSB with finite function class]{theorem}{lowerBoundCCSB}
\label{thm:lower-bound-ccsb}
For any $m, A, T, N \in \NN$ such that $A/m \in \NN$, $N \geq A/m$, $T / \lfloor \log_{A/m} N \rfloor \in \NN$, and $T / \lfloor \log_{A/m} N \rfloor \geq 16 A / m$, there exists a CCSB problem with $A$ arms and combinatorial-action size at most $m$, and a reward function class $\Fcal$ with $\abs{\Fcal} \leq N$, such that for any algorithm $\Alg$ there is an environment realizable with respect to $\Fcal$ on which the expected regret of $\Alg$ is at least $\ilowboundlog{\sqrt{m\,A\,T \log N}}$.
\end{restatable}

The construction (deferred to Appendix~\ref{sec:appendix-lower-bound})
partitions the horizon into $M = \lfloor \log_{A/m} N \rfloor$ equal-length intervals, assigns each interval a distinct context, and embeds an independent non-contextual $m$-path instance of \citet{kveton2015tight} into each interval. Aggregating the per-interval $\ilowbound{\sqrt{m A T / M}}$ regret across the $M$ intervals yields the displayed bound.

\mypara{Minimax optimality.}
Setting $N = \abs{\Fcal}$ in \Cref{thm:lower-bound-ccsb} yields a lower bound of $\ilowboundlog{\sqrt{m\,A\,T \log\abs{\Fcal}}}$, which matches the upper bound of \Cref{corol:regret-sqcb-comb} up to polylogarithmic factors. Hence \sqcomb\ attains the minimax-optimal regret rate for CCSB with general reward function approximation up to logarithmic factors. To our knowledge, this is the first minimax-optimal regret guarantee for CCSB beyond the linear-reward setting.
We also remark that our lower bound crucially uses the dependence of $(r_t(a))_{a \in \Acal}$. Under the additional assumption that $(r_t(a))_{a \in \Acal}$ are independent given $x_t$ (as studied in~\citep{combes2015combinatorial}), we conjecture that the minimax regret bound may be strictly lower than $\iupboundlog{\sqrt{m\,A\,T \log\abs{\Fcal}}}$; we leave this as an interesting open question.

%% file: Contents/5.experiments.tex
\section{Experiments}
\label{sec:experiments}

We compare \sqcomb\ against several semi-bandit and contextual-bandit baselines on four experimental settings: a non-contextual $m$-set problem, a synthetic contextual grid-path problem and two
public learning-to-rank tasks.
For each setting and algorithm configuration, we select the hyperparameter that maximizes the mean realized cumulative reward over $10$ runs, then evaluate it on another $10$ runs.
Our experimental implementation is available on GitHub.\footnote{\url{https://github.com/MjolnirT/Contextual-Combinatorial-Semi-Bandit}}

\mypara{Synthetic settings.} The non-contextual problem selects $m\in\{1,\ldots,10\}$ out of $A=10$ Bernoulli arms with fixed means. The contextual grid-path problem selects a path on a $5\times 5$ grid to maximize the cumulative reward collected along its edges. The grid contains $A=40$ directed down/right edges and $\binom{8}{4}=70$ feasible paths, each of length $m=8$, with context-dependent Bernoulli rewards. Both experiments use $T=30{,}000$ rounds, with $10$ tuning seeds and $10$ disjoint evaluation seeds. Full settings and results appear in Appendix~\ref{sec:appendix-experiments}, including \Cref{tab:mset-regret,tab:grid-path-reward} and \Cref{fig:mset-curves,fig:appendix-curves}.

\mypara{Contextual $m$-set recommendation settings.}
We use MSLR-WEB30k~\citep{qin2013letor} and the Yahoo!\ Learning-to-Rank Challenge Set~1~\citep{chapelle2011yahoo}. Both corpora are recast as CCSB instances via the standard supervised-to-bandit reduction~\citep{krishnamurthy2016contextual,qin2014contextual,foster2020beyond}, with per-document relevance labels in $\{0,1,2,3,4\}$ used as semi-bandit feedback. Following the conventions of \citet{krishnamurthy2016contextual}, we set the candidate pool size $A = 10$ and the slate size $m = 3$ on MSLR-WEB30k, and $A = 6$, $m = 2$ on Yahoo!\ LTR Set~1. After filtering out queries with fewer than $A$ candidate documents, each run uses $T = 30{,}846$ queries
on MSLR-WEB30k and all $T = 27{,}630$ retained queries on Yahoo!\ LTR Set~1.

\mypara{Algorithms.}
For the $m$-set recommendation problems, we instantiate \sqcomb\ (\Cref{alg:sqcomb}) with a linear regression (\texttt{lin}) and a gradient-boosted regression tree ensemble (\texttt{gb2}, \texttt{gb5}; see Appendix~\ref{sec:exp-algorithms} for details) as its regression oracles.
We compare against four other algorithms:
\textsc{SquareCB.Lin}~\citep{foster2020adapting}, which uses the summed reward feedback from one slate as bandit reward feedback, and constructs the sampling distribution over the $\binom{A}{m}$ size-$m$ combinatorial actions by solving a log-determinant optimization problem, and draws the slate from the resulting distribution;
\textsc{VCEE}~\citep{krishnamurthy2016contextual}, a policy-based contextual combinatorial semi-bandit algorithm;
\textsc{LinUCB}~\citep{chu11contextual}, the linear UCB algorithm adapted to the semi-bandit setting; and
$\varepsilon$-greedy~\citep{langford2007epoch}, which learns a policy from importance-weighted semi-bandit observations and, at each round, plays a uniformly random size-$m$ slate with probability $\varepsilon_t=\max\{t^{-1/2},\varepsilon\}$ or
follows its learned policy otherwise.
We additionally report two non-learning algorithms for reference: a uniform-random slate (the floor), and a supervised skyline (the ceiling) that fits a regression oracle once on the entire labeled corpus and then greedily plays the in-sample top-$m$ arms every round; the table (\Cref{tab:main-final}) reports one skyline per regression oracle (\texttt{lin}, \texttt{gb2}, \texttt{gb5}).

\mypara{Results.} For the learning-to-rank experiments, we plot the per-round average reward in \Cref{fig:main} using the \texttt{gb5} oracle, except for the linear-only \textsc{LinUCB} baseline.
We compute the mean over $10$ random seeds, with $\pm 1$ sample standard deviation. From Figure~\ref{fig:main}, \sqcomb\ (gb5) achieves the highest final mean reward among all online learners on both datasets. In this comparison, \textsc{SquareCB.Lin} achieves a lower final mean reward than the other online learners. We give further experimental results with other regression oracles, as well as more detailed discussion on the experimental results in Appendix~\ref{sec:appendix-curves}.

\begin{figure}[t]
    \centering
    \includegraphics[width=\linewidth]{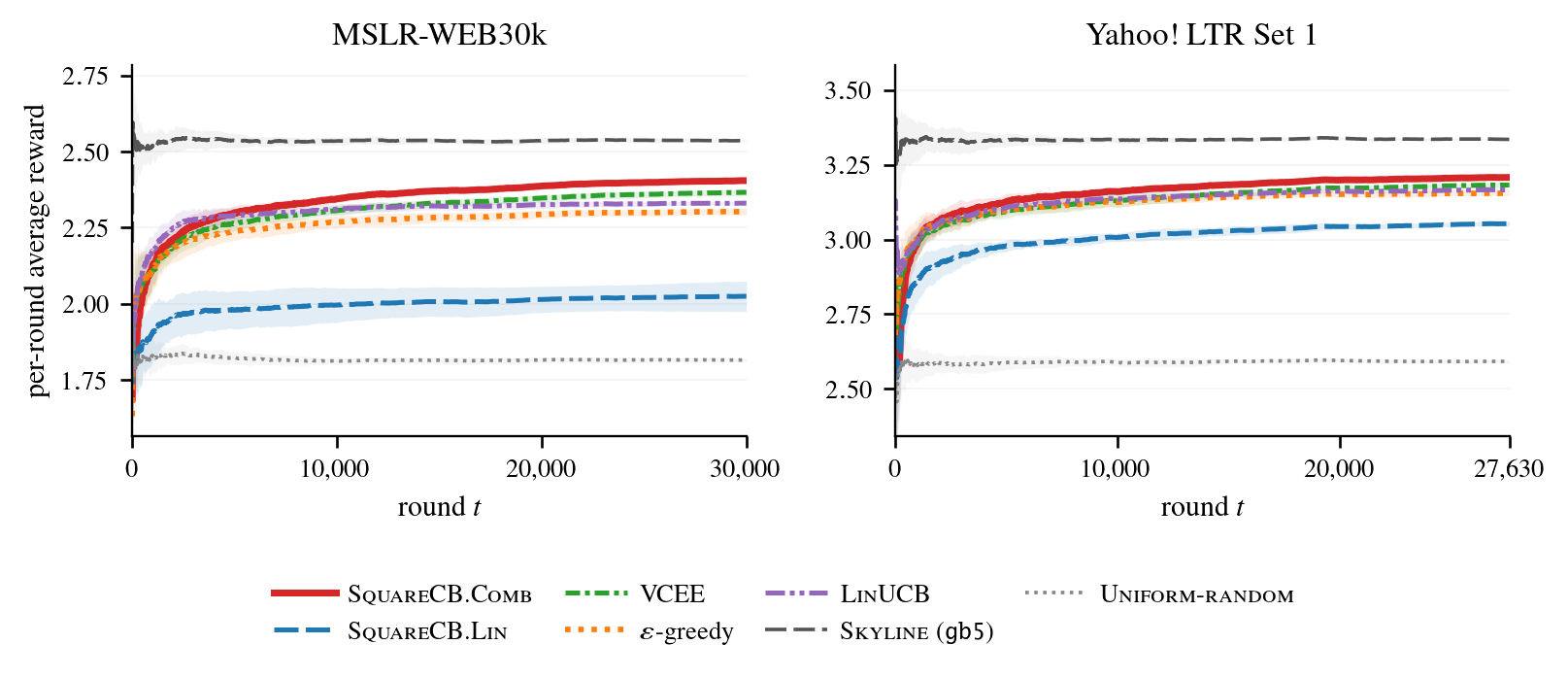}
    \caption{Per-round average reward on MSLR-WEB30k (left) and Yahoo!\ LTR Set~1 (right), with each online learner instantiated on the depth-$5$ gradient-boosted regression-tree (\texttt{gb5}) oracle (except \textsc{LinUCB}, which is linear-only). Curves are the means over $10$ seeds with $\pm 1$ standard deviation. The supervised \textsc{Skyline} (\texttt{gb5}, the ceiling) and \textsc{Uniform-random} (the floor) references bracket the online algorithms.
    }
    \label{fig:main}
\end{figure}

%% file: Contents/9.0.appendix.tex
\tableofcontents

\input{Contents/9.0.dmso.tex}
\input{Contents/9.1.squarecb-comb.tex}
\input{Contents/9.2.oe2d-comb.tex}
\input{Contents/9.3.full-bandit-regret.tex}
\input{Contents/9.4.experiments.tex}
\input{Contents/9.5.lower-bound.tex}
\input{Contents/9.x.supportings.tex}

%% file: Contents/9.0.dmso.tex
\section{Decision Making with Structured Observations}
\label{sec:dmso}

\citet{foster2021statistical} introduced a general framework for interactive decision making with structured observations (DMSO), which captures a wide range of problems including contextual bandits and reinforcement learning, and may be used for the contextual combinatorial semi-bandit setting. The interactive protocol of DMSO is as follows. At the beginning, the learner has access to a model class $\Mcal$ that contains the true model $M^\star$; each model $M$ is associated with a reward function $f^M: \Pi \to \RR$. At every time step $t$:
\begin{enumerate}
    \item The learner selects a decision $\pi_t$ from a decision set $\Pi$;
    \item The environment draws an observation $o_t \sim M^\star(\pi_t)$ and sends it to the learner; the learner incurs reward $f^{M^\star}(\pi_t)$.
\end{enumerate}

The exploration-exploitation trade-off is a fundamental challenge in DMSO, and a key step in our analysis is to characterize this quantitatively. The \textit{Decision-Estimation Coefficient} (DEC) of \citet{foster2021statistical} captures this complexity:
\[
\dec_\gamma(\Mcal, \hat M)
\coloneqq
\inf_{p \in \Delta(\Pi)} \sup_{q \in \Delta(\Pi),\, M^\star \in \Mcal}
\EE_{\pi \sim q} \sbr{ f^{M^\star}(\pi) }
- \EE_{\pi \sim p} \sbr{ f^{M^\star}(\pi) }
- \gamma \EE_{\pi \sim p} \sbr{ D_H^2 (\hat{M}(\pi), M^\star(\pi)) },
\]
where $D_H^2$ is the squared Hellinger distance between the distributions over observations induced by $\pi$ under $M$ and $\hat M$.

%% file: Contents/9.1.squarecb-comb.tex
\section{Additional Related Work}
\label{sec:addl-related-work}

\paragraph{Noncontextual combinatorial semi-bandit: stochastic and adversarial settings.}

The non-contextual combinatorial semi-bandit, also known as the combinatorial multi-armed bandit (CMAB), is a special case of the CCSB problem in which no contextual information is available.
The CMAB problem admits a variety of formulations in the literature~\citep{audibert2011minimax}.
\citet{anantharam_asymptotically_1987} first studied a special case of the CMAB problem under the fixed-cardinality constraint, in which the learner chooses $m$ arms per round, and proposed an algorithm with an asymptotically optimal regret guarantee.
\citet{gai2012combinatorial} studied the CMAB problem with arbitrary combinatorial constraints on the action space and proposed a UCB-based algorithm with an $\iupboundlog{m^3 A/\Delta^2}$ instance-dependent bound, where $\Delta$ is the minimum gap between the optimal and a suboptimal combinatorial action;
\citet{kveton2015tight} analyzed a similar UCB-based algorithm and improved the instance-dependent bound to $\iupboundlog{m A/\Delta}$ to achieve optimality when the arm set has dependency and obtained a $\iupbound{\sqrt{m A T}}$ minimax regret bound.
\citet{combes2015combinatorial} gives an optimal instance-dependent rate as $\iupboundlog{\sqrt{m} A/\Delta}$ when arms are independent.
\citet{chen2016combinatorial} extended the analysis to a general reward structure of a combinatorial action and gave both gap-dependent and gap-independent regret bounds.
\citet{jourdan2021efficient} studied the pure-exploration / best-arm identification problem with fixed confidence in the CMAB setting, where the learner aims to identify the optimal combinatorial action with high confidence.

Another line of work studies the CMAB problem in the adversarial setting with semi-bandit feedback, where each arm's reward is chosen by an adversary.
\citet{gyorgy07theonline} studied the online shortest path problem with semi-bandit feedback and adversarially chosen edge costs in a directed acyclic graph. They formalized the problem as an adversarial combinatorial MAB problem and proposed an Exp3-style exponential-weighting algorithm to achieve a regret bound of $\iupboundlog{m\sqrt{A T}}$.
\citet{neu2016importance} proposed a FTPL-based
algorithm for adversarial combinatorial semibandits with general action sets to achieve the same regret bound of $\iupboundlog{m\sqrt{A T}}$; their algorithm is  computationally efficient if the combinatorial action set has a linear optimization oracle.
\citet{audibert14regret} adapted the Online Stochastic Mirror Descent (OSMD) algorithm to solve the adversarial CMAB problem and obtained optimal minimax regret bounds $\iupbound{\sqrt{m A T}}$;
\citet{zimmert2019beating} proposed an FTRL-based algorithm that achieves the first best-of-both-worlds guarantee, with $\iupbound{\sqrt{T}}$ regret in the adversarial setting and $\iupbound{\log T}$ regret in the stochastic setting.

\paragraph{Complexity measures of online decision-making problems.}
Our method follows the Estimation-to-Decisions (E2D) paradigm of \citet{foster2021statistical}, which reduces an online decision-making problem to a statistical estimation problem. \citet{foster2021statistical} introduced the Decision-Estimation Coefficient (DEC) as a complexity measure for the interactive decision-making problems and developed the Decision Making with Structured Observations (DMSO) framework around it.
This framework has inspired a line of work that follows the same reduction idea.
For example, \citet{neu2024optimistic} bridged the Bayesian theory of information-directed sampling and the frequentist theory of E2D via the DEC; \citet{kirschner2023regret} relaxed E2D to an anytime variant and applied it to linear bandits with side observations.
Beyond the DEC, \citet{xu2025bayesian} proposed a Bayesian variant of the DEC and developed an adaptive posterior sampling algorithm for adversarial DMSO frameworks.
\citet{zhang2022feel} introduces the decoupling coefficient that controls the dependence between a sampled reward model and its greedy action.
\citet{lattimore2020exploration} proposes the ExO algorithm for partial monitoring problems that jointly optimize an exploration distribution and a loss estimator.
\citet{foster2022complexity} then proves a generalized ExO objective is equivalent to the convexified DEC.

\section{Regret Guarantees of \sqcomb\ in Mode 1}
\label{sec:sqcomb-mode1}
In this section, we present and prove the regret guarantee of \sqcomb\ in Mode 1, and provide all necessary lemmas with their proofs in the following subsection. Mode 2 is analyzed in Appendix~\ref{sec:sample-linear-opt-oracle}.
\subsection{Lemmas about the reduction from DEC to the log-barrier optimization}

The analysis in this subsection is conducted conditional on the observed context $x_t$ and the estimated reward function $\hat{f}_t$ at round $t$, and we omit the dependence on $x_t$ and $\hat{f}_t$ for brevity. We write $g(a)$, $\hat{g}(a)$, and $\Gcal$ to denote $f^\star(x_t, a)$, $\hat{f}_t(x_t, a)$, and $\Fcal_{x_t} := \cbr{ f(x_t, \cdot): f \in \Fcal }$, respectively.
In the proof sketch of \Cref{thm:regret-sqcb-comb-regsq}, we noted that the solution of the log-barrier optimization problem in Eq.~\eqref{eqn:log-barrier} certifies that the CS-DEC is upper bounded by $\frac{A}{\gamma}$ up to a constant factor.
\begin{align*}
    \bar{p}_t = \arg \max_{\bar{q} \in \bar{\Scal}} \inner{\hat{g}}{\bar{q}} + \frac{1}{\gamma} \sum_{a \in \Acal} \log(\bar{q}(a))
\end{align*}
Specifically, we aim to show that the solution $\bar{p}_t$ satisfies the following inequality:
\begin{align}
    \max_{\bar{q} \in \bar{\Scal}, g^\star} \inner{g^\star}{\bar{q}} - \inner{g^\star}{\bar{p}_t} - \gamma \cdot \opnorm{g^\star - \hat{g}}{\bar{p}_t}^2
    \lesssim
    \frac{A}{\gamma},
        \label{eqn:small-dec}
\end{align}
To this end, we introduce a surrogate min-max optimization problem that connects the log-barrier optimization problem with Eq.~\eqref{eqn:small-dec}.
Specifically, the surrogate min-max optimization problem is defined as follows:
\begin{align}
    \min_{\bar{p} \in \bar{\Scal}} \max_{\bar{q} \in \bar{\Scal}} \inner{\hat{g}}{\bar{q}} - \inner{\hat{g}}{\bar{p}} + \frac{1}{\gamma} \coverage(\bar{p}, \bar{q})
        \label{eqn:small-doec}
\end{align}
where the coverage term is defined as $\coverage(\bar{p}, \bar{q}) = \sum_{a \in \Acal} \frac{\bar{q}(a)^2}{\bar{p}(a)}$.
We show in \Cref{lemma:log-barrier-certify-doec} that the solution $\bar{p}_t$ of the log-barrier optimization problem approximately solves Eq.~\eqref{eqn:small-doec}.
Taking Eq.~\eqref{eqn:small-doec} as an intermediate result, we then show in \Cref{lemma:doec-to-dec} that $\bar{p}_t$ certifies a small CS-DEC value (Eq.~\eqref{eqn:small-dec}).

\begin{lemma}[Solving log-barrier regularized problem approximately solves Eq.~\eqref{eqn:small-doec}]
    \label{lemma:log-barrier-certify-doec}
    For any $\hat{g}$ %
    , define
    \[
        \Lcal(\bar{s}; \hat{g}) := \inner{\hat{g}}{\bar{s}} + \frac{1}{\gamma} \sum_{a \in \Acal} \log(\bar{s}(a))
    \]
    for any $\bar{s} \in \bar{\Scal}$. Let $\hat{p}$ be the optimal solution of $\max_{\bar{s} \in \bar{S}} \Lcal(\bar{s}; \hat{g})$,
    then we have
    \[
        \max_{\bar{q} \in \bar{\Scal}} \inner{\hat{g}}{\bar{q}} - \inner{\hat{g}}{\hat{p}} + \frac{1}{\gamma} \coverage(\hat{p}, \bar{q})
        \leq
        \frac{A}{\gamma}.
    \]
\end{lemma}
\begin{proof}[Proof of \Cref{lemma:log-barrier-certify-doec}]
    Throughout this proof, for any vector $u \in \RR^{A}$ with strictly positive entries, we write $\frac{1}{u} \in \RR^{A}$ for its entrywise reciprocal, and $u^2 \in \RR^A$ for its entrywise square.

    Since $\Lcal(\bar{q}; g)$ is a strictly concave function of $\bar{q}$ and $\bar{\Scal}$ is a convex set, we know that the optimal solution $\hat{p}$ is achieved when it meets the first-order optimality condition, that is, for any $\bar{q} \in \bar{\Scal}$, we have
    \begin{align*}
        \inner{\nabla_{\bar{q}} -\Lcal(\bar{q})\mid_{\bar{q}=\hat{p}}}{\bar{q} - \hat{p}} &\geq 0 \\
        \Leftrightarrow \inner{-\hat{g} - \frac{1}{\gamma \hat{p}}}{\bar{q} - \hat{p}} &\geq 0
            \tag{Gradient of $-\Lcal$ at $\hat{p}$}
        \\
        \Leftrightarrow \inner{\hat{g} + \frac{1}{\gamma \hat{p}}}{\bar{q} - \hat{p}} &\leq 0
        \\
        \Leftrightarrow \inner{\hat{g}}{\bar{q} - \hat{p}} + \frac{1}{\gamma} \inner{\frac{1}{\hat{p}}}{\bar{q}} &\leq \frac{A}{\gamma}
    \end{align*}

    Since $\hat{p}$ satisfies the above first-order optimality condition, we know that for any $\bar{q} \in \bar{\Scal}$,
    \begin{align}
        \inner{\hat{g}}{\bar{q} - \hat{p}} + \frac{1}{\gamma} \inner{\frac{1}{\hat{p}}}{\bar{q}} &\leq \frac{A}{\gamma} \nonumber \\
        \Rightarrow \inner{\hat{g}}{\bar{q}} - \inner{\hat{g}}{\hat{p}} + \frac{1}{\gamma} \inner{\frac{1}{\hat{p}}}{\bar{q}^2} &\leq \frac{A}{\gamma} \nonumber \\
        \Leftrightarrow \inner{\hat{g}}{\bar{q}} - \inner{\hat{g}}{\hat{p}} + \frac{1}{\gamma} \coverage(\hat{p}, \bar{q}) &\leq \frac{A}{\gamma}
            \label{eqn:log-barrier-optimality}
    \end{align}

\end{proof}

\begin{lemma}[Certifying Eq.~\eqref{eqn:small-dec} by Eq.~\eqref{eqn:small-doec}]
    \label{lemma:doec-to-dec}
    Let $\hat{g}$ be any reward estimate and $\hat{p} \in \bar{\Scal}$ a marginal probability vector. If $\hat{p}$ satisfies the inequality
    \[
        \max_{\bar{q} \in \bar{\Scal}} \inner{\hat{g}}{\bar{q}} - \inner{\hat{g}}{\hat{p}} + \frac{1}{\gamma} \coverage(\hat{p}, \bar{q})
        \leq
        \frac{A}{\gamma},
    \]
    then $\hat{p}$ also satisfies the inequality
    \[
        \max_{\bar{q} \in \bar{\Scal},\, g^\star} \inner{g^\star}{\bar{q}} - \inner{g^\star}{\hat{p}} - \gamma \cdot \opnorm{g^\star - \hat{g}}{\hat{p}}^{2}
        \lesssim
        \frac{A}{\gamma},
    \]
    where we recall that the squared weighted norm is defined as $\opnorm{g^\star - \hat{g}}{\hat{p}}^{2} = \sum_{a \in \Acal} \hat{p}(a)\,\bigl(g^\star(a) - \hat{g}(a)\bigr)^{2}$.
\end{lemma}
\begin{proof}[Proof of \Cref{lemma:doec-to-dec}]
    Fix the ground-truth reward function $g^\star$ and two marginal probability vectors $\bar{s}, \hat{p} \in \bar{\Scal}$. Decompose the excess reward among $\bar{s}$ and $\hat{p}$ as follows:
    \begin{align*}
        \inner{g^\star}{\bar{s}} - \inner{g^\star}{\hat{p}}
        = \inner{\hat{g}}{\bar{s}} - \inner{\hat{g}}{\hat{p}}
        + \inner{g^\star - \hat{g}}{\bar{s}}
        + \inner{\hat{g} - g^\star}{\hat{p}}
    \end{align*}
    We bound the second difference by using the Off-policy evaluation (OPE) lemma (\Cref{lemma:ope-error}) followed by AM-GM:
    \begin{align*}
        \inner{g^\star - \hat{g}}{\bar{s}}
        &\leq
        \sqrt{ \coverage(\hat{p}, \bar{s}) \cdot \opnorm{g^\star - \hat{g}}{\hat{p}}^2}
        \\
        &\lesssim
        \frac{1}{\gamma} \coverage(\hat{p}, \bar{s}) + \gamma \cdot \opnorm{g^\star- \hat{g}}{\hat{p}}^2.
    \end{align*}
    For the third term, Cauchy-Schwarz combined with $\sum_{a \in \Acal} \bar{s}(a) \leq m, \forall \bar{s} \in \bar{\Scal}$ from \Cref{assum:action-size} and AM-GM yields
    \begin{align*}
        \inner{\hat{g} - g^\star}{\hat{p}}
        &\leq
        \sqrt{\idel{\sum_{a\in\Acal}\hat{p}(a)} \cdot \opnorm{\hat{g} - g^\star}{\hat{p}}^2}
        \\
        &\leq
        \frac{m}{\gamma} + \frac{\gamma}{4} \cdot \opnorm{\hat{g} - g^\star}{\hat{p}}^2.
    \end{align*}
    Plugging the two bounds back into the decomposition,
    \begin{align*}
        \inner{g^\star}{\bar{s}} - \inner{g^\star}{\hat{p}}
        \lesssim \inner{\hat{g}}{\bar{s}} - \inner{\hat{g}}{\hat{p}} + \frac{1}{\gamma} \coverage(\hat{p}, \bar{s}) + \frac{m}{\gamma} + \gamma \cdot \opnorm{g^\star - \hat{g}}{\hat{p}}^2.
    \end{align*}
    Since the above inequality holds for any $\bar{s} \in \bar{\Scal}$, we can take the maximum over $\bar{s}$ on the left-hand side and use the fact that $\hat{p}$ approximately solves Eq.~\eqref{eqn:small-doec} to upper bound the sum of first three terms on the right-hand side by $\frac{A}{\gamma}$, which gives
    \begin{align*}
        \inner{g^\star}{\bar{s}} - \inner{g^\star}{\hat{p}}
        &\lesssim \frac{A}{\gamma} + \frac{m}{\gamma} + \gamma \cdot \opnorm{g^\star - \hat{g}}{\hat{p}}^2
        \\
        \Rightarrow \inner{g^\star}{\bar{s}} - \inner{g^\star}{\hat{p}} - \gamma \cdot \opnorm{g^\star - \hat{g}}{\hat{p}}^2
        &\lesssim \frac{A}{\gamma}
    \end{align*}
    The last inequality holds for any $\bar{s} \in \bar{\Scal}$ and $g^\star \in \Gcal$, which completes the proof.
\end{proof}

\begin{lemma}[Solution of log-barrier optimization certifies small DEC]
    \label{lemma:certify-small-dec}
    Under \Cref{assum:realizable,assum:action-size}, the solution $\bar{p}_t$ of the log-barrier optimization problem in Eq.~\eqref{eqn:log-barrier} satisfies that for any $\bar{q} \in \bar{\Scal}$ and any $g^\star \in \Gcal$, we have
    \begin{align*}
        \inner{g^\star}{\bar{q}} - \inner{g^\star}{\bar{p}_t} - \gamma \cdot \opnorm{g^\star - \hat{g}}{\bar{p}_t}^2
        \lesssim
        \frac{A}{\gamma}.
    \end{align*}
\end{lemma}
\begin{proof}[Proof of \Cref{lemma:certify-small-dec}]
    From \Cref{lemma:log-barrier-certify-doec}, we have that the solution $\bar{p}_t$ of the log-barrier optimization problem in Eq.~\eqref{eqn:log-barrier} certifies Eq.~\eqref{eqn:small-doec}.
    Then, applying \Cref{lemma:doec-to-dec} by letting $\hat{p} = \bar{p}_t$, we get that $\bar{p}_t$ certifies Eq.~\eqref{eqn:small-dec}.
\end{proof}

\subsection{Regret guarantee of \sqcomb\ in Mode 1}

We are now ready to bound the regret of \sqcomb\ at each round by leveraging the fact that the solution of the log-barrier optimization problem certifies a small DEC value. We define the regret of \sqcomb\ at round $t$ as $\regret_t(\bar{p}_t) = \max_{\bar{q} \in \bar{\Scal}} \inner{f^\star(x_t, \cdot)}{\bar{q}} - \inner{f^\star(x_t, \cdot)}{\bar{p}_t}$, which represents the difference between the expected reward of the best combinatorial action and that of the combinatorial action selected by \sqcomb\ at round $t$.
\begin{lemma}[Instantaneous regret upper bound of \sqcomb]
    \label{lemma:regret-per-round}
        Under \Cref{assum:realizable,assum:action-size}, for each round, the regret of \sqcomb\ in Mode 1 is bounded by
        \begin{align}
            \regret_t(\bar{p}_t) \lesssim \frac{A}{\gamma} + \gamma \opnorm{f^\star(x_t, \cdot) - \hat{f}_t(x_t, \cdot)}{\bar{p}_t}^2.
        \end{align}
\end{lemma}
\begin{proof}[Proof of \Cref{lemma:regret-per-round}]
    By \Cref{lemma:certify-small-dec}, the solution $\bar{p}_t$ of the log-barrier optimization problem in Eq.~\eqref{eqn:log-barrier} certifies that, for any $\bar{q} \in \bar{\Scal}$ and any $f^\star \in \Fcal$,
    \[
        \inner{f^\star(x_t, \cdot)}{\bar{q}} - \inner{f^\star(x_t, \cdot)}{\bar{p}_t} - \gamma \cdot \opnorm{f^\star(x_t, \cdot) - \hat{f}_t(x_t, \cdot)}{\bar{p}_t}^2
        \lesssim
        \frac{A}{\gamma}.
    \]
    Then, for the regret of \sqcomb\ at round $t$, we have
    \begin{align*}
        &\quad \regret_t(\bar{p}_t)
        \\
        &=
        \max_{\bar{q} \in \bar{\Scal}} \inner{f^\star(x_t, \cdot)}{\bar{q}} - \inner{f^\star(x_t, \cdot)}{\bar{p}_t}
        \\
        &=
        \max_{\bar{q} \in \bar{\Scal}} \inner{f^\star(x_t, \cdot)}{\bar{q}} - \inner{f^\star(x_t, \cdot)}{\bar{p}_t} - \gamma \cdot \opnorm{f^\star(x_t, \cdot) - \hat{f}_t(x_t, \cdot)}{\bar{p}_t}^2 + \gamma \cdot \opnorm{f^\star(x_t, \cdot) - \hat{f}_t(x_t, \cdot)}{\bar{p}_t}^2
        \\
        &\leq \frac{A}{\gamma} + \gamma \cdot \opnorm{f^\star(x_t, \cdot) - \hat{f}_t(x_t, \cdot)}{\bar{p}_t}^2,
    \end{align*}
    where the first inequality is due to Eq.~\eqref{eqn:small-dec}.
\end{proof}

We are now ready to prove the main regret guarantee of \sqcomb.

\begin{proof}[Proof of \Cref{thm:regret-sqcb-comb-regsq} for Mode 1]
    By \Cref{lemma:regret-per-round}, we bound the instantaneous regret of \sqcomb\ at every round $t$. Summing over $t = 1, \ldots, T$ and then taking expectation over all the randomness we have
    \begin{align*}
        \EE[\Regret(T)]
        &= \sum_{t=1}^T \EE[\regret_t(\bar{p}_t)]
        \\
        &\leq \sum_{t=1}^T \frac{A}{\gamma} + \gamma \cdot \EE[\opnorm{f^\star(x_t, \cdot) - \hat{f}_t(x_t, \cdot)}{\bar{p}_t}^2]
            \tag{Apply \Cref{lemma:regret-per-round}}
        \\
        &= \frac{T A}{\gamma} + \gamma \cdot \EE \sbr{ \EE_{s_t \sim p_t} \sbr{ \sum_{t=1}^T \sum_{a: s_t(a) = 1} \del{f^\star(x_t, a) - \hat{f}_t(x_t, a)}^2 } }
        \\
        &\leq \frac{T A}{\gamma} + \gamma\, \Reg_{\mathrm{batch}}(T)
    \end{align*}
    In the second-to-the-last equation, we use the definition of expectation.
    The last inequality uses the online regression guarantee in \Cref{assum:online-reg-semi}, where for any sequence of $(s_t, a_t, o_t)_{t=1}^T$, we have
    \begin{align*}
    & \EE\sbr{ \sum_{t=1}^T \sum_{a: s_t(a) = 1} (\hat{f}(x_t, a) - f^\star(x_t, a))^2 } \\
    = &
    \EE\sbr{ \sum_{t=1}^T \sum_{a: s_t(a) = 1}
    \rbr{ (\hat{f}(x_t, a) - r_t(a))^2 -  (f^\star(x_t, a) - r_t(a))^2 }
    }
    \\
    \leq &  \Reg_{\mathrm{batch}}(T),
    \end{align*}
    where the equality is due to that $\EE\sbr{r_t(a) \mid x_t } = f^\star(x_t, a)$.

    Setting $\gamma = \sqrt{A T / \Reg_{\mathrm{batch}}(T)}$ balances the dominant exploitation term $T A/\gamma$ against the exploration penalty $\gamma \Reg_{\mathrm{batch}}(T)$, yielding $\Regret(T) \leq \Ocal\del{\sqrt{A T \Reg_{\mathrm{batch}}(T)}}$, which completes the proof.
\end{proof}

\subsection{Anytime exploration}
\label{sec:sqcomb-anytime}

\begin{theorem}[Anytime regret bound]
\label{thm:regret-sqcb-comb-anytime}
Suppose \Cref{assum:realizable,assum:action-size,assum:online-reg-semi} hold and $2\leq|\Fcal|<\infty$. Run Algorithm~\ref{alg:sqcomb} in Mode 1 with exploration parameter $\gamma_t=\sqrt{At/(m\log|\Fcal|)}$ in Eq.~\eqref{eqn:log-barrier} at each round $t\geq1$. Then, for every horizon $T\geq1$, the expected regret of \sqcomb\ is bounded as
\[
    \EE\sbr{\Regret(T)}\leq\Ocal\del{\sqrt{mAT\log|\Fcal|}}.
\]
\end{theorem}

\begin{proof}[Proof of \Cref{thm:regret-sqcb-comb-anytime}]
Fix $T\geq1$. The regret decomposition becomes
\begin{align*}
    \EE\sbr{\Regret(T)}
    &=\sum_{t=1}^T\EE\sbr{\regret_t(\bar p_t)}
    \\
    &\lesssim\sum_{t=1}^T\left(
        \frac{A}{\gamma_t}
        +\gamma_t\EE\sbr{\opnorm{f^\star(x_t,\cdot)-\hat f_t(x_t,\cdot)}{\bar p_t}^{2}}
    \right)
    \\
    &\leq\sum_{t=1}^T\frac{A}{\gamma_t}
        +\gamma_T\sum_{t=1}^T\EE\sbr{\opnorm{f^\star(x_t,\cdot)-\hat f_t(x_t,\cdot)}{\bar p_t}^{2}}
    \\
    &\lesssim\sum_{t=1}^T\frac{A}{\gamma_t}+\gamma_T m\log|\Fcal|
    \\
    &=\sqrt{mA\log|\Fcal|}\sum_{t=1}^T\frac{1}{\sqrt t}
        +\sqrt{mAT\log|\Fcal|}
    \\
    &\lesssim\sqrt{mAT\log|\Fcal|}.
\end{align*}
The first inequality applies Lemma~\ref{lemma:regret-per-round} with $\gamma=\gamma_t$. The second uses $\gamma_t\leq\gamma_T$ and nonnegativity of the squared prediction errors. The third uses the expected prediction-error bound from the proof of Theorem~\ref{thm:regret-sqcb-comb-regsq}, which follows from realizability and \Cref{assum:online-reg-semi}. The last inequality uses $\sum_{t=1}^T t^{-1/2}\leq2\sqrt T$. Thus the same regret rate holds for every horizon, while the schedule does not require $T$ as input.
\end{proof}

The experimental schedule $\gamma_t=\gamma_0\sqrt{At/m}$ in Appendix~\ref{sec:exp-algorithms} has the same dependence on $t$. The theoretical choice above corresponds to $\gamma_0=1/\sqrt{\log|\Fcal|}$, whereas the experiments tune $\gamma_0$.

\section{Convergence and Regret Guarantee of \sqcomb\ in Mode 2}
\label{sec:sample-linear-opt-oracle}

We now analyze Mode 2 of Algorithm~\ref{alg:sqcomb}, which calls Algorithm~\ref{alg:sample-soln} to jointly compute a marginal probability vector and a sampling distribution using a linear optimization oracle over $\Scal$. Our goal is to show that the returned distribution certifies a small CS-DEC objective, as required by the regret analysis (see Lemma~\ref{lemma:semibandit-approximate-dec} below).

Algorithm~\ref{alg:sample-soln} takes into an estimated reward vector $\hat g_t$, an exploration parameter $\gamma>0$ and cushion parameter $\varepsilon\in(0,1/2]$ to the algorithm, which returns a distribution $p_t$ over combinatorial actions that satisfies Eq.~\eqref{eqn:semibandit-approximate-dec}.
In addition, we would like $p_t$ to be sparsely supported over a small set of combinatorial actions so that it can be efficiently sampled.

    We assume access to an efficient linear optimization oracle over the combinatorial action set $\Scal$, which is a mild assumption and has been adopted in prior works~\citet{neu2016importance}.
    Without such an oracle, even the much simpler problem of offline, noncontextual combinatorial optimization with known rewards may be intractable. For example, under knapsack constraints $\Scal = \icbr{s \in \icbr{0,1}^A: \sum_{a \in \Acal} w(a)s(a) \leq c}$, where $w(a)$ is the weight of arm $a$ and $c$ is the capacity, exact maximization of a known linear reward is the NP-hard $0$--$1$ knapsack problem~\citep{cacchiani2022knapsack}.

Our Algorithm~\ref{alg:sample-soln} is a generalization of \citet[][Algorithm~2]{qin2026taming} to the combinatorial semi-bandit setting (see also~\citep{agarwal2014taming,xu2020upper,zhu2022contextualb} for related ideas of efficiently finding an exploration distribution using optimization oracles).
The algorithm first computes the greedy action $\hat s$ and initializes all action masses to zero. Each iteration selects an action using predicted rewards plus inverse-marginal bonuses. If the selected action's predicted reward plus exploration bonus exceeds the greedy action's predicted reward plus the threshold, the algorithm adds mass to this action using an adaptive step size and updates the marginals. Otherwise, it stops and assigns the remaining mass to $\hat s$, and returns the resulting sampling distribution.

In a nutshell, Algorithm~\ref{alg:sample-soln} aims at finding distribution $p_t$ whose marginal probability vector $\bar{p}_t$ certifies the following surrogate min-max objective to be small:
\begin{equation}
    \inf_{\bar p\in\bar{\Scal}}\max_{\bar q\in\bar{\Scal}}
    \cbr{\inner{\hat g_t}{\bar q-\bar p}
    +\frac{1}{\gamma}\overline{\coverage}_{\varepsilon}(\bar p,\bar q)}.
    \label{eqn:semibandit-surrogate}
\end{equation}
This is a slight variant of the surrogate objective in Eq.~\eqref{eqn:doec} by replacing $\coverage(\bar p,\bar q)$ with $\overline{\coverage}_{\varepsilon}(\bar p,\bar q)$:
\begin{equation}
    \overline{\coverage}_{\varepsilon}(\bar p,\bar q)
    \coloneqq \sum_{a \in \Acal}\frac{\bar q(a)}{\bar p(a)+\varepsilon/A},
    \qquad \bar p \in \RR_+^A,\quad \bar q \in \bar{\Scal}.
    \label{eqn:semibandit-relaxed-coverage}
\end{equation}
For fixed $\bar p$, this objective is linear in $\bar q$, so the inner maximization can be implemented by a linear optimization oracle over $\Scal$.
Specifically, we show in Lemma~\ref{thm:semibandit-linear-oracle} (deferred to the next subsection) that \Cref{alg:sample-soln} efficiently finds $p_t$ that certifies the objective in~\eqref{eqn:semibandit-relaxed-coverage} is $\tilde{O}\rbr{ \frac{A}{\gamma} }$.
This, along with a little algebra, gives the main guarantee for~\Cref{alg:sample-soln}, stated below:

\begin{lemma}[Approximate CS-DEC bound]
\label{lemma:semibandit-approximate-dec}
For $\gamma\geq\sqrt{2}$ and $\varepsilon=\gamma^{-2}$,the distribution $p_t$ returned by Algorithm~\ref{alg:sample-soln} has marginal probability vector $\bar p_t=\EE_{s\sim p_t}[s]$ satisfying
\begin{equation}
    \label{eqn:semibandit-approximate-dec}
    \max_{\substack{\bar q\in\bar{\Scal}\\g^\star\in[0,1]^A}}
    \cbr{\inner{g^\star}{\bar q-\bar p_t}
    -\gamma\opnorm{g^\star-\hat g_t}{\bar p_t}^{2}}
    \leq\frac{12A\log(1+A/\varepsilon)}{\gamma}.
\end{equation}
\end{lemma}

\begin{proof}[Proof of Lemma~\ref{lemma:semibandit-approximate-dec}]
We generalize the argument of \citet[Theorem~4]{qin2026taming} from bandit setting to semi-bandit setting. Fix $\bar q\in\bar{\Scal}$ and $g^\star\in[0,1]^A$, and let $h=g^\star-\hat g_t\in[-1,1]^A$.
First, we start from a modified version of OPE upper bound in the following form:

\begin{align*}
    \abs{\inner{h}{\bar q}}
    &\leq\sqrt{\overline{\coverage}_{\varepsilon}(\bar p_t,\bar q)
        \left(\varepsilon+\opnorm{h}{\bar p_t}^{2}\right)}
            \tag{By Proposition~\ref{prop:semibandit-coverage-error}}
    \\
    &\leq\frac{\overline{\coverage}_{\varepsilon}(\bar p_t,\bar q)}{2\gamma}
        +\frac{\gamma}{2}\left(\varepsilon+\opnorm{h}{\bar p_t}^{2}\right),
            \tag{By Young's inequality}
    \\
    \abs{\inner{h}{\bar p_t}}
    &\leq\sqrt{\norm{\bar p_t}_1}\opnorm{h}{\bar p_t}
        \leq\sqrt{m}\opnorm{h}{\bar p_t}
        \leq\frac{m}{2\gamma}+\frac{\gamma}{2}\opnorm{h}{\bar p_t}^{2}.
\end{align*}
Here Young's inequality is $uv\leq u^2/(2\gamma)+\gamma v^2/2$. Since $\inner{h}{\bar q}-\inner{h}{\bar p_t}\leq\abs{\inner{h}{\bar q}}+\abs{\inner{h}{\bar p_t}}$, substituting these bounds cancels the two $\frac{\gamma}{2}\opnorm{h}{\bar p_t}^{2}$ terms against $-\gamma\opnorm{h}{\bar p_t}^{2}$, yielding

\begin{align*}
    \inner{g^\star}{\bar q-\bar p_t}-\gamma\opnorm{h}{\bar p_t}^{2}
    &=\inner{\hat g_t}{\bar q-\bar p_t}
      +\inner{h}{\bar q}-\inner{h}{\bar p_t}
      -\gamma\opnorm{h}{\bar p_t}^{2}
    \\
    &\leq\inner{\hat g_t}{\bar q-\bar p_t}
      +\frac{\overline{\coverage}_{\varepsilon}(\bar p_t,\bar q)+m}{2\gamma}
      +\frac{\gamma\varepsilon}{2}
    \\
    &\leq\frac{12A\log(1+A/\varepsilon)}{\gamma}.
\end{align*}
The last inequality uses termination condition (Eq.~\eqref{eqn:semibandit-output-certificate}) and nonnegativity of the relaxed coverage, with $\varepsilon=\gamma^{-2}$. Taking the maximum over $\bar q$ and $g^\star$ proves the claim.
\end{proof}

\begin{algorithm}[H]
{
\caption{Computing exploration distribution of combinatorial actions (Mode 2 of Algorithm~\ref{alg:sqcomb})}
\label{alg:sample-soln}
\begin{algorithmic}[1]
\STATE \textbf{Input:} estimated reward vector $\hat g_t \in[0,1]^A$, exploration parameter $\gamma>0$, cushion parameter $\varepsilon\in(0,1/2]$, linear optimization oracle over $\Scal$.
\STATE \textbf{Output:} a distribution $p_t$ over $\Scal$, stored as a weighted combination of at most $\lfloor 32A\log(1+A/\varepsilon)\rfloor+1$ combinatorial actions.
\STATE Compute the greedy action $\hat s\in\argmax_{s\in\Scal}\inner{\hat g_t}{s}$ \COMMENT{one linear-oracle call}
\STATE Initialize $p^0\leftarrow0$, the zero measure over $\Scal$; its marginal probability vector is $\bar p^0=\mathbf{0}\in\RR^A$.
\FOR{$j=1,2,\ldots$}
    \STATE Compute $s^j\in\argmax_{s\in\Scal}\inner{\hat g_t+\left(\frac{1}{\gamma(\bar p^{j-1}(a)+\varepsilon/A)}\right)_{a=1}^A}{s}$ \COMMENT{one linear-oracle call}
    \IF{$\inner{\hat g_t}{s^j-\hat s}+\frac{1}{\gamma}\sum_{a:s^j(a)=1}\frac{1}{\bar p^{j-1}(a)+\varepsilon/A}>\frac{8A}{\gamma}\log\left(1+\frac{A}{\varepsilon}\right)$}
        \STATE Set $\Delta_j\leftarrow\frac{1}{4}\left(\sum_{a:s^j(a)=1}\frac{1}{\bar p^{j-1}(a)+\varepsilon/A}\right)^{-1}$; update $p^j\leftarrow p^{j-1}+\Delta_j\delta_{s^j}$, so that $\bar p^j=\bar p^{j-1}+\Delta_j s^j$ \COMMENT{adaptive stepsize update}
    \ELSE
        \STATE \textbf{Return} $p_t=p^{j-1}+(1-\norm{p^{j-1}}_1)\delta_{\hat s}$
    \ENDIF
\ENDFOR
\end{algorithmic}
}
\end{algorithm}

\subsection{Convergence analysis}
\label{sec:semibandit-convergence}

The following theorem shows that Algorithm~\ref{alg:sample-soln} terminates after a finite number of iterations, and its output $p_t$ is a probability distribution on $\Scal$ supported on at most $N_\varepsilon+1$ actions. Its marginal probability vector $\bar p_t$ approximately satisfies an upper bound of Eq.~\eqref{eqn:doec}; more precisely, it satisfies
Eq.~\eqref{eqn:semibandit-output-certificate}
that can be converted into a bound on $\bar{p}_t$'s CS-DEC objective.

\begin{lemma}[Termination guarantee for Algorithm~\ref{alg:sample-soln}]
\label{thm:semibandit-linear-oracle}
Let $\gamma>0$, $\varepsilon\in(0,1/2]$, and
\[
    N_{\varepsilon}\coloneqq\left\lfloor32A\log\left(1+\frac{A}{\varepsilon}\right)\right\rfloor.
\]
Algorithm~\ref{alg:sample-soln} terminates after $K<N_{\varepsilon}$ accepted updates, with $\norm{p^K}_1<1$. Its output $p_t$ is a probability distribution on $\Scal$, supported on at most $N_{\varepsilon}+1$ actions, and its marginal probability vector $\bar p_t$ satisfies
\begin{equation}
    \max_{\bar q\in\bar{\Scal}}\cbr{
        \inner{\hat g_t}{\bar q-\bar p_t}
        +\frac{1}{\gamma}\overline{\coverage}_{\varepsilon}(\bar p_t,\bar q)
    }
    \leq\frac{10A}{\gamma}\log\left(1+\frac{A}{\varepsilon}\right).
    \label{eqn:semibandit-output-certificate}
\end{equation}
\end{lemma}

\subsubsection{Properties of the relaxed coverage}
\label{sec:semibandit-coverage-properties}
We follow the proof strategy of \citet{qin2026taming} to establish the properties of the relaxed coverage, $\overline{\coverage}_{\varepsilon}$, defined in Eq.~\eqref{eqn:semibandit-relaxed-coverage}. We establish some basic structural properties and the dilution stability property for the termination analysis of Algorithm~\ref{alg:sample-soln}.

\begin{proposition}[Basic structural properties]
\label{prop:semibandit-coverage-structure}
For fixed $\bar p\in\RR_+^A$ and $\varepsilon>0$, the relaxed coverage is linear in $\bar q\in\bar{\Scal}$. It is nonincreasing in each coordinate of $\bar p$ and in $\varepsilon$, and jointly continuous on $\RR_+^A\times\bar{\Scal}\times(0,\infty)$. Moreover, for every $c>0$,
\begin{equation}
    \overline{\coverage}_{c\varepsilon}(c\bar p,\bar q)
    =\frac{1}{c}\overline{\coverage}_{\varepsilon}(\bar p,\bar q).
    \label{eqn:semibandit-coverage-scaling}
\end{equation}
\end{proposition}

\begin{proof}
Linearity follows because the denominators in Eq.~\eqref{eqn:semibandit-relaxed-coverage} are fixed when $\bar p$ and $\varepsilon$ are fixed. Increasing $\bar p$ coordinatewise or increasing $\varepsilon$ increases those denominators, while every numerator is nonnegative. Continuity follows from positivity of the denominators. Scaling both $\bar p$ and $\varepsilon$ by $c$ multiplies each denominator by $c$, proving Eq.~\eqref{eqn:semibandit-coverage-scaling}.
\end{proof}

\begin{proposition}[Dilution stability]
\label{prop:semibandit-coverage-dilution}
Let $\varepsilon>0$, $\bar p\in\RR_+^A$, and $s\in\Scal$ satisfy $C\coloneqq\overline{\coverage}_{\varepsilon}(\bar p,s)>0$. For $\Delta=(4C)^{-1}$,
\begin{equation}
    \overline{\coverage}_{\varepsilon}(\bar p+\Delta s,s)
    \geq \frac{3}{4}\overline{\coverage}_{\varepsilon}(\bar p,s).
    \label{eqn:semibandit-coverage-dilution}
\end{equation}
\end{proposition}

\begin{proof}
For $\Delta=\frac{1}{4}\overline{\coverage}_{\varepsilon}(\bar p,s)^{-1}$, we have
\begin{align*}
    &\overline{\coverage}_{\varepsilon}(\bar p+\Delta s,s)
    \\
    =&\sum_{a:s(a)=1}\frac{1}{\bar p(a)+\varepsilon/A+\Delta}
    \\
    \geq&\sum_{a:s(a)=1}\left(
        \frac{1}{\bar p(a)+\varepsilon/A}
        -\frac{\Delta}{(\bar p(a)+\varepsilon/A)^2}\right)
    \\
    \geq&\overline{\coverage}_{\varepsilon}(\bar p,s)
        -\Delta\,\overline{\coverage}_{\varepsilon}(\bar p,s)^2
    \\
    =&\frac{3}{4}\overline{\coverage}_{\varepsilon}(\bar p,s).
\end{align*}
The first inequality uses $1/(x+\Delta)\geq1/x-\Delta/x^2$ for $x,\Delta>0$; the second uses $\sum_a y_a^2\leq(\sum_a y_a)^2$ for nonnegative vector $y$; the last equality uses $\Delta=\frac{1}{4}\overline{\coverage}_{\varepsilon}(\bar p,s)^{-1}$. Thus, one step of Algorithm~\ref{alg:sample-soln} shrinks the coverage in its own update direction to no less than $3/4$ of its value.
\end{proof}

\subsubsection{Potential function and its properties}
\label{sec:semibandit-potential-properties}

To bound the number of updates in Algorithm~\ref{alg:sample-soln}, we introduce a potential combining accumulated coverage and estimated regret. We show that it starts at zero (\Cref{prop:semibandit-zero-initialization}), decreases by at least a constant at each accepted update (\Cref{prop:semibandit-constant-decrease}), and remains bounded below while the total mass is at most one (\Cref{prop:semibandit-potential-lower-bound}), providing the key ingredients for the termination proof.

For any finite nonnegative measure $p$ on $\Scal$, define its estimated regret by
\[
    \widehat R(p)
    \coloneqq \sum_{s\in\Scal}p(s)\inner{\hat g_t}{\hat s-s}
    =\norm{p}_1\inner{\hat g_t}{\hat s}-\inner{\hat g_t}{\bar p}.
\]
The greedy choice of $\hat s$ ensures $\widehat R(p)\geq0$. Following the potential-based argument of \citet{qin2026taming}, define
\begin{align}
    F_j&\coloneqq\sum_{i=1}^{j}\Delta_i\overline{\coverage}_{\varepsilon}(\bar p^i,s^i),
    \qquad
    G_j\coloneqq\frac{\gamma}{2}\sum_{i=1}^{j}\Delta_i\widehat R(\delta_{s^i})
    =\frac{\gamma}{2}\widehat R(p^j),
    \nonumber\\
    \Phi_j&\coloneqq-F_j+G_j.
    \label{eqn:semibandit-potential}
\end{align}
Here the coverage in $F_j$ is evaluated at the marginal probability vector after each update. We establish the following three properties of this potential.

\begin{proposition}[Zero-initialization property]
\label{prop:semibandit-zero-initialization}
We have $F_0=G_0=\Phi_0=0$, and $G_j\geq0$ for every completed update $j$.
\end{proposition}

\begin{proof}
At $j=0$, the sums defining $F_j$ and $G_j$ are empty. Each summand in $G_j$ is nonnegative because $\Delta_i>0$ and $\widehat R(\delta_{s^i})\geq0$.
\end{proof}

\begin{proposition}[Constant-decreasing property]
\label{prop:semibandit-constant-decrease}
For every accepted update $j$ of Algorithm~\ref{alg:sample-soln},
\begin{equation}
    \Phi_{j-1}-\Phi_j\geq\frac{1}{16}.
    \label{eqn:semibandit-constant-decrease}
\end{equation}
\end{proposition}

\begin{proof}
The update test implies
\[
    \overline{\coverage}_{\varepsilon}(\bar p^{j-1},s^j)
    >\gamma\widehat R(\delta_{s^j})+8A\log\left(1+\frac{A}{\varepsilon}\right)
    \geq\gamma\widehat R(\delta_{s^j}).
\]
Using Proposition~\ref{prop:semibandit-coverage-dilution} and then this inequality gives
\begin{align*}
    \Phi_{j-1}-\Phi_j
    &=\Delta_j\overline{\coverage}_{\varepsilon}(\bar p^j,s^j)
      -\frac{\gamma}{2}\Delta_j\widehat R(\delta_{s^j})
    \\
    &\geq\frac{3\Delta_j}{4}\overline{\coverage}_{\varepsilon}(\bar p^{j-1},s^j)
      -\frac{\gamma}{2}\Delta_j\widehat R(\delta_{s^j})
    \\
    &\geq\frac{\Delta_j}{4}\overline{\coverage}_{\varepsilon}(\bar p^{j-1},s^j)
    =\frac{1}{16},
\end{align*}
where the last equality uses the adaptive step size in Algorithm~\ref{alg:sample-soln}.
\end{proof}

\begin{proposition}[Lower-boundedness property]
\label{prop:semibandit-potential-lower-bound}
For every $j\geq0$ such that $\norm{p^j}_1\leq1$,
\begin{equation}
    \Phi_j\geq-F_j
    \geq-A\log\left(1+\frac{m}{\varepsilon}\right)
    \geq-A\log\left(1+\frac{A}{\varepsilon}\right).
    \label{eqn:semibandit-potential-lower-bound}
\end{equation}
\end{proposition}

\begin{proof}
Since $G_j\geq0$, we have $\Phi_j\geq-F_j$. The update $\bar p^i=\bar p^{i-1}+\Delta_i s^i$ gives
\begin{align*}
    F_j
    &=\sum_{i=1}^{j}\sum_{a:s^i(a)=1}
        \frac{\Delta_i}{\bar p^i(a)+\varepsilon/A}
    \\
    &\leq\sum_{i=1}^{j}\sum_{a\in\Acal}
        \log\left(\frac{\bar p^i(a)+\varepsilon/A}{\bar p^{i-1}(a)+\varepsilon/A}\right)
    \\
    &=\sum_{a\in\Acal}\log\left(1+\frac{A\bar p^j(a)}{\varepsilon}\right)
    \\
    &\leq A\log\left(1+\frac{\norm{\bar p^j}_1}{\varepsilon}\right)
    \leq A\log\left(1+\frac{m}{\varepsilon}\right).
\end{align*}
The first inequality uses $u/(x+u)\leq\log((x+u)/x)$ for $x,u>0$; unselected arms contribute zero logarithmic increments. The next equality telescopes over the updates and uses $\bar p^0=0$. The last two inequalities follow from Jensen's inequality and $\norm{\bar p^j}_1=\sum_{s\in\Scal}p^j(s)\norm{s}_1\leq m\norm{p^j}_1\leq m$, respectively. Finally, $m\leq A$ gives the stated bound.
\end{proof}

\subsubsection{Termination and the output certificate}
\label{sec:semibandit-termination}

\begin{proof}[Proof of Theorem~\ref{thm:semibandit-linear-oracle}]
The update test and step-size rule give $0<\Delta_j< \frac 1 {32A\log(1+A/\varepsilon)}$
for every accepted update. If $N_{\varepsilon}$ updates were completed, then $\norm{p^{N_{\varepsilon}}}_1< \frac{N_{\varepsilon}}{32A\log(1+A/\varepsilon)}\leq1$.
Proposition~\ref{prop:semibandit-potential-lower-bound} therefore gives the lower bound
\[
    \Phi_{N_{\varepsilon}}
    \geq-A\log\left(1+\frac{A}{\varepsilon}\right).
\]
On the other hand, Propositions~\ref{prop:semibandit-zero-initialization} and~\ref{prop:semibandit-constant-decrease} give the upper bound
\[
    \Phi_{N_{\varepsilon}}
    \leq-\frac{N_{\varepsilon}}{16}
    <-A\log\left(1+\frac{A}{\varepsilon}\right),
\]
where the strict inequality uses the definition of $N_{\varepsilon}$ and $A\log(1+A/\varepsilon)>1$. This contradicts the lower bound. Thus the algorithm terminates with $K<N_{\varepsilon}$ and $\norm{p^K}_1<1$. Its output $p_t=p^K+(1-\norm{p^K}_1)\delta_{\hat s}$ is consequently a probability distribution supported on at most $K+1\leq N_{\varepsilon}+1$ actions.

At termination, the failed update test implies that for every $\bar q\in\bar{\Scal}$,
\begin{equation}
    \inner{\hat g_t}{\bar q-\hat s}
    +\frac{1}{\gamma}\overline{\coverage}_{\varepsilon}(\bar p^K,\bar q)
    \leq\frac{8A}{\gamma}\log\left(1+\frac{A}{\varepsilon}\right).
    \label{eqn:semibandit-coverage-before-topup}
\end{equation}
Adding the remaining mass at $\hat s$ preserves estimated regret and increases the marginal probability vector coordinatewise. Since $\Phi_K\leq0$ and Proposition~\ref{prop:semibandit-potential-lower-bound} upper bounds $F_K$, we obtain
\begin{align}
    \widehat R(p_t)
    &=\widehat R(p^K)
    =\frac{2(\Phi_K+F_K)}{\gamma}
    \leq\frac{2A}{\gamma}\log\left(1+\frac{A}{\varepsilon}\right),
    \label{eqn:semibandit-low-estimated-regret}
    \\
    \overline{\coverage}_{\varepsilon}(\bar p_t,\bar q)
    &\leq\overline{\coverage}_{\varepsilon}(\bar p^K,\bar q),
    \nonumber
\end{align}
where the second inequality follows from monotonicity. Combining these bounds with Eq.~\eqref{eqn:semibandit-coverage-before-topup} and $\widehat R(p_t)=\inner{\hat g_t}{\hat s-\bar p_t}$ proves Eq.~\eqref{eqn:semibandit-output-certificate}.
\end{proof}

\subsubsection{End-to-end computational cost}
\label{sec:semibandit-computation}

Let $\mathcal T_{\mathrm{lin}}$ be the time required to solve $\max_{s\in\Scal}\inner{\ell}{s}$ for any $\ell\in\RR^A$, and let $\mathcal T_{\mathrm{reg}}$ bound the per-round cost of the batch online regression oracle, including one update on the newly observed data and $A$ evaluations of $\hat f_t(x_t,\cdot)$.

By Theorem~\ref{thm:semibandit-linear-oracle}, each invocation of Algorithm~\ref{alg:sample-soln} accepts $K=\iupboundlog{TA}$ updates and makes $K+2$ linear optimization oracle calls: one for the greedy action and one in each of the $K+1$ iterations, including the final iteration in which the stopping condition is met. Each iteration requires $\Ocal(A)$ additional arithmetic operations to form the oracle objective, check the stopping condition, and update the marginal probability vector and stored weighted action. Sampling directly from the at most $K+1$ stored weights takes $\Ocal(K+1)$ time. Thus, over $T$ rounds, Algorithm~\ref{alg:sqcomb} with Mode 2 makes $T$ batch regression-oracle calls and $\iupboundlog{TA}$ linear optimization oracle calls, with total running time
\begin{equation}
    \upboundlog{T\mathcal T_{\mathrm{reg}}+TA(\mathcal T_{\mathrm{lin}}+A)}.
    \label{eqn:semibandit-total-computation}
\end{equation}
With the choices of $\varepsilon=\gamma^{-2}$ and $\gamma$ as in Theorem~\ref{thm:semibandit-linear-oracle-regret}, this becomes $\iupboundlog{(T\mathcal T_{\mathrm{reg}}+TA(\mathcal T_{\mathrm{lin}}+A))}$.

For unordered $m$-sets, the linear oracle selects the top-$m$ coordinates of $\ell$, which takes $\mathcal T_{\mathrm{lin}}=\Ocal(A)$ time using linear-time selection. For contextual shortest path on a DAG with $N$ vertices and $M=A$ edges, the oracle uses one dynamic-programming pass in topological order, taking $\mathcal T_{\mathrm{lin}}=\Ocal(N+A)$ time for arbitrary signed edge weights. This is $\Ocal(A)$ when $N=\Ocal(A)$, as holds after removing isolated vertices. In either case with $\mathcal T_{\mathrm{lin}}=\Ocal(A)$, the total cost is $\Ocal(T\mathcal T_{\mathrm{reg}}+TA^2\log(1+AT))$.

\subsection{Regret of \sqcomb\ in Mode 2}
\label{sec:semibandit-linear-oracle-regret}

We use the output certificate of Theorem~\ref{thm:semibandit-linear-oracle} to prove the following regret guarantee.

\begin{theorem}[Regret of \sqcomb\ in Mode 2]
\label{thm:semibandit-linear-oracle-regret}
Suppose \Cref{assum:realizable,assum:action-size,assum:online-reg-semi} hold, and the batch online regression oracle outputs predictions in $[0,1]$. Run Algorithm~\ref{alg:sqcomb} in Mode 2, sampling $s_t\sim p_t$, and parameters $\varepsilon=\gamma^{-2}$ and $\gamma=\sqrt{AT/(m\ln\abs{\Fcal})}$. Then
\begin{equation}
    \EE\sbr{\Regret(T)}
    \leq\upboundlog{\sqrt{mAT\ln\abs{\Fcal}}}.
    \label{eqn:semibandit-linear-oracle-regret}
\end{equation}
\end{theorem}

\begin{proof}[Proof of Theorem~\ref{thm:semibandit-linear-oracle-regret}]
The expected squared prediction error is bounded by the regression oracle's excess square loss:
\begin{align*}
    &\sum_{t=1}^{T}\EE\sbr{\opnorm{f^\star(x_t,\cdot)-\hat f_t(x_t,\cdot)}{\bar p_t}^{2}}
    \\
    &\quad=\EE\left[\sum_{t=1}^{T}\sum_{a:s_t(a)=1}
        \cbr{(\hat f_t(x_t,a)-r_t(a))^2
        -(f^\star(x_t,a)-r_t(a))^2}\right]
    \\
    &\quad\leq\EE\sbr{\Reg_{\mathrm{batch}}(T)}
        \lesssim m\log|\Fcal|.
\end{align*}
The equality uses $\bar p_t(a)=\EE_{s_t\sim p_t}[s_t(a)]$ and expansion of the square losses: the cross term $2(\hat f_t(x_t,a)-f^\star(x_t,a))(r_t(a)-f^\star(x_t,a))$ has zero conditional expectation before observing the rewards. The first inequality uses $f^\star\in\Fcal$; the last is \Cref{assum:online-reg-semi} with at most $m$ observations per round.

Applying Lemma~\ref{lemma:semibandit-approximate-dec} with $g^\star=f^\star(x_t,\cdot)$, summing over rounds, and taking expectations therefore gives
\[
    \EE\sbr{\Regret(T)}
    \leq\frac{\bigl(12A\log(1+A/\varepsilon)\bigr)T}{\gamma}
    +\gamma\EE\sbr{\Reg_{\mathrm{batch}}(T)}.
\]
Since we set $\varepsilon=\gamma^{-2}$ and $\gamma=\sqrt{AT/m}$, the above bound becomes
\[
    \EE\sbr{\Regret(T)}\lesssim \upboundlog{\sqrt{mAT\ln\abs{\Fcal}}}.
\]
\end{proof}

The following property of the relaxed coverage supports the preceding CS-DEC bound.

\begin{proposition}[Validity of the relaxed coverage]
\label{prop:semibandit-coverage-error}
For any $\varepsilon>0$, distributions $p,q\in\Delta(\Scal)$ with marginal probability vectors $\bar p,\bar q$, and $g\in[-1,1]^A$,
\begin{equation}
    \frac{\inner{g}{\bar q}^{2}}
        {\varepsilon+\sum_{a\in\Acal}\bar p(a)g(a)^2}
    \leq\overline{\coverage}_{\varepsilon}(\bar p,\bar q).
    \label{eqn:semibandit-coverage-error}
\end{equation}
\end{proposition}

\begin{proof}
For any $p\in\Delta(\Scal)$ and $g\in\RR^A$, we have $\EE_{s\sim p}[\sum_{a:s(a)=1}g(a)]=\inner{g}{\bar p}$. For $g\in[-1,1]^A$ (instantiated as $f^\star(x_t,\cdot)-\hat f_t(x_t,\cdot)$ in the regret analysis), first suppose $g(a)\neq0$ for every $a\in\Acal$. Starting from the cushioned error ratio, we have
\begin{align*}
    \frac{\inner{g}{\bar q}^{2}}
        {\varepsilon+\sum_{a\in\Acal}\bar p(a)g(a)^2}
    &\leq
    \frac{\left(\sum_{a\in\Acal}g(a)\bar q(a)\right)^2}
        {\sum_{a\in\Acal}(\bar p(a)+\varepsilon/A)g(a)^2}
    \\
    &\leq
    \sum_{a\in\Acal}\frac{g(a)^2\bar q(a)^2}
        {(\bar p(a)+\varepsilon/A)g(a)^2}
    \\
    &\leq
    \sum_{a\in\Acal}\frac{\bar q(a)}{\bar p(a)+\varepsilon/A}
    =\overline{\coverage}_{\varepsilon}(\bar p,\bar q).
\end{align*}
The first inequality uses $g(a)^2\leq1$, so $(\varepsilon/A)\sum_{a\in\Acal}g(a)^2\leq\varepsilon$; the second is Cauchy--Schwarz; the last uses $\bar q(a)\in[0,1]$, so $\bar q(a)^2\leq\bar q(a)$. This proves the stated ratio bound; the case of zero coordinates of $g$ follows by continuity.
\end{proof}

%% file: Contents/9.2.oe2d-comb.tex
\section{An Offline Oracle-efficient Variant: \oetdcomb}
\label{sec:oe2d-comb}

In this section, we present \oetdcomb, an epoch-based variant of \sqcomb\ that replaces the batch-mode online regression oracle $\batchoracle$ with an offline regression oracle $\offoracle$, in the spirit of the offline oracle-efficient contextual bandit algorithms FALCON~\citep{simchi2022bypassing} and the OE2D framework of \citet{qin2026taming}. The algorithm partitions the horizon $T$ into $M$ epochs with endpoints $0 = \tau_0 < \tau_1 < \cdots < \tau_M = T$; within each epoch $m$, the reward estimate $\hat f_m \in \Fcal$ is held fixed, and at every round the same log-barrier optimization problem (the analogue of Eq.~\eqref{eqn:log-barrier} with $\gamma$ replaced by $\gamma_m$) is solved to obtain a marginal probability vector $\bar p_t \in \bar\Scal$, and we sample a combinatorial action using the sampling oracle according to $\bar p_t$; this is the same as Mode 1 of \sqcomb; we may also compute the exploration distribution as in Mode 2 of \sqcomb -- details are omitted.
We assume access to an offline regression oracle $\offoracle$ that, given a batch of i.i.d.\ semi-bandit observations, returns the empirical risk minimizer $\hat f_m \in \Fcal$ of the squared loss summed over the observed coordinates; its on-policy estimation error is controlled by \Cref{lemma:off-reg-semi-bandit}.

\begin{algorithm}[t]
    \begin{algorithmic}
    \STATE \textbf{Input:} Semi-bandit tuple $(\Xcal, \Scal, \Acal)$, horizon $T$, offline regression oracle $\offoracle$, sampling oracle $\sampleoracle$, epoch schedule $\cbr{\tau_m}_{m=1}^M$, exploration schedule $\cbr{\gamma_m}_{m=1}^M$, reward function class $\Fcal$.
    \FOR{epoch $m = 1, 2, \ldots, M$}
        \STATE Let $\hat{f}_m = 0$ for $m=1$. Otherwise,
        Call $\offoracle$ to obtain $\hat f_m$ from the dataset $\cbr{(x_\tau, s_\tau, o_\tau)}_{\tau = \tau_{m-2}+1}^{\tau_{m-1}}$ collected in epoch $m-1$;
        \FOR{round $t = \tau_{m-1}+1, \ldots, \tau_m$}
            \STATE Observe context $x_t$.
            \STATE Compute the marginal probability vector
            \begin{align*}
                \bar{p}_t = \argmax_{\bar p \in \bar{\Scal}}
                \inner{\hat{f}_m(x_t, \cdot)}{\bar p} + \frac{1}{\gamma_m}\sum_{a \in \Acal} \log\del{\bar{p}(a)}.
            \end{align*}
            \STATE Sample a combinatorial action $s_t$ such that $\EE\sbr{s_t} = \bar{p}_t$ by calling $\sampleoracle(\bar{p}_t, \Scal)$.

            \STATE Observe semi-bandit feedback $o_t = \cbr{r_t(a) : s_t(a) = 1}$.
        \ENDFOR
    \ENDFOR
    \end{algorithmic}
    \caption{OE2D-comb: epoch-based offline-regression variant of \sqcomb.}
    \label{alg:oe2d-comb}
\end{algorithm}

The optimization problem solved at every round of \oetdcomb\ is precisely the $A$-dimensional log-barrier optimization problem in \sqcomb\ (Eq.~\eqref{eqn:log-barrier}), instantiated with the per-epoch estimate $\hat f_m$ and exploration parameter $\gamma_m$. As in \sqcomb, the convex hull constraint $\bar p_t \in \bar\Scal$ ensures that $\bar p_t$ is realizable by some distribution over $\Scal$, and the sampling oracle $\sampleoracle$ converts $\bar p_t$ into a combinatorial action $s_t$ with matching marginals (cf. Section~\ref{sec:reduction}). By \Cref{lemma:log-barrier-certify-doec}, the resulting $\bar p_t$ certifies the surrogate min-max objective of Eq.~\eqref{eqn:doec} at value at most $A/\gamma_m$, which we use below to prove \oetdcomb's regret bound:

\begin{theorem}
    \label{thm:oe2d-comb-restate}
    Suppose that the contexts $\cbr{x_t}_{t=1}^T$ are drawn i.i.d.\ from $\Dcal_\Xcal$. Under \Cref{assum:realizable,assum:action-size}, \oetdcomb\ with the doubling schedule $\tau_m = 2^m$ and exploration parameter
    \[
        \gamma_m = \sqrt{A / \Reg_{\offoracle}(\Fcal, \tau_{m-1}/2)}
    \]
    satisfies
    \begin{align*}
        \EE\sbr{\Regret(T)}
        \leq \widetilde{\Ocal}\del{\sqrt{m\, A\, T \log\abs{\Fcal}}}.
    \end{align*}
\end{theorem}

\begin{proof}[Proof sketch of \Cref{thm:oe2d-comb-restate}]

    Our proof utilizes Theorem~\ref{thm:abstract-oe2d} (deferred after the proof), a regret theorem for an epoch-based decision making algorithm abstracted from~\citet{qin2026taming}.
    First, we will adopt a policy-viewpoint of \oetdcomb:
    at epoch $m$, we view the algorithm as executing a policy $\pi(x)$ with marginal probability vector:
    \[
    \bar{\pi}_m(x)
    =
    \argmax_{\bar p \in \bar{\Scal}}
    \inner{\hat{f}_m(x, \cdot)}{\bar p} + \frac{1}{\gamma_m}\sum_{a \in \Acal} \log\del{\bar{p}(a)};
    \]
    Therefore, the regret of the algorithm can be represented as
    \[
    \EE\sbr{\Regret(T)}
    =
    \sum_{m=1}^M (\tau_m - \tau_{m-1}) \EE_{x \sim \Dcal_\Xcal} \sbr{ \max_{\lambda \in \Lambda} \inner{f^\star(x,\cdot)}{\lambda} - \inner{f^\star(x,\cdot)}{\pi_m(x)} }.
    \]

    We now apply Theorem~\ref{thm:abstract-oe2d} by taking $\Lambda = \bar{\Scal}$, $p_m = \bar{\pi}_m$,
    $C = \coverage$ (recall that
    $\coverage(\bar{p}, \bar{q}) = \sum_{a=1}^A \frac{\bar{q}(a)^2}{\bar{p}(a)}$), $D_m = \frac{A}{\gamma_m}$ and $E_m = \opnorm{\hat{f}_m - f^\star}{\bar{p}_{m-1}}^2$.
    The two conditions of Theorem~\ref{thm:abstract-oe2d} are satisfied:
    \begin{itemize}
    \item Condition~\ref{item:p-m-doec} follows from Lemma~\ref{lemma:log-barrier-certify-doec};

    \item Condition~\ref{item:ope} follows from Lemma~\ref{lemma:ope-error};

    \end{itemize}

    Thus, applying the theorem,
    \begin{align*}
        &\EE\sbr{\Regret(T)}  \\
        \leq & \EE\sbr{\tau_1 + \del{\max_{n \in [2, M]} \frac{\tau_n}{\gamma_n}} \del{A + \max_{n \in [2, M]} \gamma_n^2 \opnorm{\hat{f}_m - f^\star}{\bar{p}_{m-1}}^2 } }
        \\
        \lesssim &
        \tau_1
        +
        \rbr{ \max_{n \in [2,M]} \frac{\tau_n}{\gamma_n} }\cdot \rbr{ A + \max_{n \in [2,M]} \gamma_n^2 \Reg_{\offoracle}(\Fcal, \tau_{n-1}/2) } \\
        \lesssim & 1
        + \max_{n \in [2,M]} \tau_n \sqrt{ \frac{\Reg_{\offoracle}(\Fcal, \tau_{n-1}/2) }{A} } \cdot A \\
        \lesssim &
        1 + \max_{n \in [2,M]} 2^n \cdot \sqrt{ \frac{ m A  \ln|\Fcal| }{2^n} }
        \lesssim
        \sqrt{m A T \ln|\Fcal|}.
    \end{align*}
    where the first inequality is by Theorem~\ref{thm:abstract-oe2d}; the second inequality uses that $\EE\isbr{\opnorm{\hat{f}_m - f^\star}{\bar{p}_{m-1}}^2} \leq \Reg_{\offoracle}(\Fcal, \tau_{m-1}/2)$ by Lemma~\ref{lemma:off-reg-semi-bandit};
    the third inequality uses $\gamma_m = \sqrt{A / \Reg_{\offoracle}(\Fcal, \tau_{m-1}/2)}$; and the fourth inequality uses that $\Reg_{\offoracle}(\Fcal, \tau_{m-1}/2) = \sqrt{ \frac{m \ln|\Fcal|}{\tau_{m-1}} }$ and the definition of $\tau_n = 2^n$; the last inequality is by algebra.
    \end{proof}

\begin{theorem}[Abstract version of~\citet{qin2026taming}, Theorem 1]
\label{thm:abstract-oe2d}
Suppose $\Lambda \subset \RR_+^A$, and
we have a sequence of reward predictors
$(\widehat{G}_m)_{m=1}^M \subset \RR^{\Acal}$
and distributions $(p_m)_{m=1}^M \subset \Lambda$,
a mapping $C: \Lambda \times \Lambda \mapsto \RR^+$,
and sequence of positive numbers
$\cbr{(D_m, E_m)}_{m=1}^M$ and a sequence of nonnegative increasing numbers $\cbr{\gamma_m}_{m=1}^M$,
that satisfies:
\begin{enumerate}
\item For every $m \geq 1$, $p_m$ is such that
\[
    \max_{\lambda \in \Lambda}\,
    \inner{\widehat{G}_m}{\lambda} - \inner{\widehat{G} _m}{p_m} + \frac{1}{\gamma_m} C (p_m, \lambda)
    \leq
    D_m
\]
\label{item:p-m-doec}
\item For every $m \geq 2$, and for any $\lambda \in \Lambda$,
\[
    \abs{ \inner{\widehat{G}_m}{\lambda}
    -
    \inner{G^\star}{\lambda}
    }
    \leq
    \sqrt{ C(p_{m-1}, \lambda) E_{m-1} }
\] \label{item:ope}
\end{enumerate}
Then, the following holds:
\begin{align*}
& \sum_{m=1}^M (\tau_m - \tau_{m-1}) \sbr{ \max_{\lambda \in \Lambda} \inner{G^\star}{\lambda} - \inner{G^\star}{p_m} } \\
\lesssim & \tau_1
+
\rbr{ \max_{n \in [2,M]} \frac{\tau_n}{\gamma_n} }\cdot \rbr{ \max_{n \in [1,M]} \gamma_n D_n + \max_{n \in [2,M]} \gamma_n^2 E_n }
\end{align*}
\end{theorem}

%% file: Contents/9.3.full-bandit-regret.tex
\section{The Suboptimality of Using Contextual Bandit Algorithms to Solve CCSB Problems}
\label{sec:cb-ccsb}

In this section, we present the regret guarantee obtained by applying \sqlin~\citep{foster2020adapting} and \falconlin~\citep{xu2020upper}, originally designed for contextual bandits with linear reward structure per-context, to the CCSB problem under full-bandit feedback. Following the canonical reduction described in \Cref{sec:related-work},
from the contextual bandit learner's perspective,
the expected aggregate reward of combinatorial action $s$ is linear in $s$ (for every fixed context $x$): $\EE\sbr{ \inner{r_t}{s_t} \mid x_t = x , s_t = s } = \inner{f^\star(x,\cdot)}{s}$. Under this reduction, the contextual bandit problem's decision space is $\Scal$, the action feature dimension is $A$, the
reward range is $[0, m]$, and the reward function class $\Hcal = \cbr{ h(x, s) = \inner{f(x, \cdot)}{s}: f \in \Fcal }$.

\subsection{Applying \sqlin\ under full-bandit feedback}
\label{sec:sqlin-full-bandit}

For the standard finite-class online square-loss regression guarantee, we refer to \citet[Theorem~3.2 and Proposition~3.1]{cesa2006prediction}.

Applying SquareCB.Lin's regret theorem~\citep[Theorem 1]{foster2020adapting} under the above reduction requires a single change relative to the standard linear contextual bandit setting: the batch-mode online regression oracle $\batchoracle$ now operates on aggregated bandit feedback $(x_t, s_t, \inner{r_t}{s_t})$, whose label range is $[0, m]$ rather than $[0, 1]$. As a result, the square-loss regret bound provided by $\batchoracle$ scales as
\[
    \Reg_{\mathrm{batch}}(T) \;\lesssim\; m^2 \log|\Fcal|,
\]
i.e., a factor of $m^2$
worse than the unit-reward case, reflecting the fact that the slate reward range is $[0, m]$. The IGW-style exploration distribution and the linear-DEC certificate $A/\gamma$ used in the proof of SquareCB.Lin is unaffected by the reduction; only the regression regret rate changes.

\begin{theorem}
\label{thm:sqlin-full-bandit}
Under \Cref{assum:realizable,assum:action-size}, \sqlin\ applied to the CCSB problem with full-bandit feedback, with constant exploration parameter $\gamma = \sqrt{\frac{A T}{m^2 \log|\Fcal|}}$, satisfies
\[
    \EE\sbr{\Regret(T)} \;\leq\; \Ocal\del{m \sqrt{A\, T\, \log|\Fcal|}}.
\]
\end{theorem}

\begin{proof}[Proof sketch of \Cref{thm:sqlin-full-bandit}]
Identical to the proof of \citet[Theorem 1]{foster2020adapting}, except that we plug the inflated regression rate $\Reg_{\mathrm{batch}}(T) \lesssim m^2 \log|\Fcal|$ into the general SquareCB.Lin bound
\[
    \EE\sbr{\Regret(T)}
    \;\lesssim\;
    \sqrt{A\, T\, \Reg_{\mathrm{batch}}(T)}
    \;\lesssim\;
    m \sqrt{A\, T\, \log|\Fcal|}.
\]
\end{proof}

\subsection{Applying \falconlin\ under full-bandit feedback}
\label{sec:falcon-lin-full-bandit}

Applying \falconlin's regret theorem~\citep[][Section 4]{xu2020upper} under the above reduction requires a single change: the offline regression oracle now operates on aggregated bandit feedback $(x_t, s_t, \inner{r_t}{s_t})$, and its on-policy guarantee is given by \Cref{lemma:off-reg-full-bandit} rather than the standard offline regression guarantee for $[0,1]$-valued rewards, with rate
\[
    \Reg_{\offoracle}(\Fcal, n) \;\lesssim\; \frac{m^2 \log|\Fcal|}{n},
\]
i.e., a factor of $m$ worse than the unit-reward case, reflecting the fact that the slate reward range is $[0, m]$. The per-context DOEC certificate $D_m = A/\gamma_m$ and the off-policy evaluation lemma used to verify Conditions~\ref{item:p-m-doec} and~\ref{item:ope} of \Cref{thm:abstract-oe2d} carry over unchanged from the standard linear contextual bandit analysis.

\begin{theorem}
\label{thm:falcon-lin-full-bandit}
Suppose contexts $\cbr{x_t}_{t=1}^T$ are drawn i.i.d.\ from $\Dcal_\Xcal$. Under \Cref{assum:realizable,assum:action-size}, \falconlin\ applied to the CCSB problem with full-bandit feedback, with the doubling schedule $\tau_m = 2^m$ and exploration parameter $\gamma_m = \sqrt{A/\Reg_{\offoracle}(\Fcal, \tau_{m-1}/2)}$, satisfies
\[
    \EE\sbr{\Regret(T)} \;\leq\; \widetilde{\Ocal}\del{m \sqrt{A\, T\, \log|\Fcal|}}.
\]
\end{theorem}

\begin{proof}[Proof sketch of \Cref{thm:falcon-lin-full-bandit}]
Identical to the proof of \citet{xu2020upper}, except that we use the guarantee from \Cref{lemma:off-reg-full-bandit},
\[
    \Reg_{\offoracle}(\Fcal, n) \lesssim m^2 \log|\Fcal|/n,
\]
in the abstract bound of \Cref{thm:abstract-oe2d}:
\begin{align*}
    \EE\sbr{\Regret(T)}
    &\lesssim
    \tau_1
    +
    \rbr{\max_{n \in [2,M]} \tfrac{\tau_n}{\gamma_n}} \cdot \rbr{\max_{n \in [1,M]} \gamma_n D_n + \max_{n \in [2,M]} \gamma_n^2 \Reg_{\offoracle}(\Fcal, \tau_{n-1}/2)} \\
    &\lesssim
    1 + \max_{n \in [2,M]} 2^n \sqrt{\tfrac{m^2 A \ln|\Fcal|}{2^n}}
    \;\lesssim\;
    m \sqrt{A\, T\, \ln|\Fcal|}.
\end{align*}
\end{proof}

%% file: Contents/9.4.experiments.tex
\section{Experiment Details}
\label{sec:appendix-experiments}
We compare our algorithm against representative existing methods for contextual combinatorial semi-bandits on four experimental settings: a non-contextual $m$-set problem, a contextual grid-path problem, and two contextual slate recommendation problems based on
MSLR-WEB30k~\citep{qin2013letor} and the Yahoo!\ Learning-to-Rank Challenge Set~1~\citep{chapelle2011yahoo}.
We treat the two learning-to-rank corpora as CCSB problems with the combinatorial action space $\Scal$ being all $m$-sets (the unordered slate setting).
Both corpora are recast as CCSB instances via the supervised-to-bandit reduction commonly used in the contextual-bandit literature~\citep{krishnamurthy2016contextual,qin2014contextual,foster2020beyond}, as we give more details next; we use the average realized reward $\frac{1}{T}\sum_{t=1}^T \sum_{a: s_t(a) = 1} r_t(a)$, where $r_t(a)$
is the recorded relevance label for document $a$ in query $x_t$ and $T$ is the horizon, as the performance metric. This is the same reduction and metric used by \citet{krishnamurthy2016contextual} and \citet{foster18practical}, so our numbers are directly comparable to theirs.

\subsection{Experimental settings}
\label{sec:exp-settings}
We begin with a non-contextual $m$-set sanity check, followed by a synthetic contextual grid-path problem and two public learning-to-rank corpora.

\begin{itemize}
    \item \textbf{Non-contextual $m$-set.}
    We use a stochastic semi-bandit with $A=10$ arms and action set
    \[
        \Scal=\cbr{s\in\{0,1\}^{A}:\sum_{a=1}^{A}s(a)=m},
        \qquad m\in\{1,\ldots,10\}.
    \]
    We set the mean reward of each arm to be $\mu(a)$, with the following fixed vector:
    \[
        \mu=(0.90,\,0.82,\,0.74,\,0.52,\,0.48,\,0.44,\,0.40,\,0.36,\,0.32,\,0.28).
    \]
    Each arm's reward is a Bernoulli random variable with the corresponding mean, and the learner receives the rewards of the selected arms as semi-bandit feedback.
    Varying $m$ provides a basic check of learning and action selection without context variation: $m=1$ recovers ordinary bandit learning, while $m=A$ has one feasible action and zero pseudo-regret.

    \item \textbf{Contextual grid-path.}
    We evaluate all algorithms on a contextual variant of the grid-path problem of \citet{kveton2015tight}. The agent traverses a $5\times5$ grid of nodes from the top-left corner $(0,0)$ to the bottom-right corner $(4,4)$, moving one step right or down at a time. The arms are the $A=40$ directed down/right edges, each action is a monotone path consisting of $m=8$ edges, and the action set contains $\binom{8}{4}=70$ feasible paths.

    Each edge $e$ is represented by a feature vector $a_e=(d_e,u_e,v_e)\in\RR^3$ encoding its direction and location. Here, $d_e=1$ for a downward edge and $d_e=0$ for a rightward edge, while $u_e,v_e\in[0,1]$ denote its normalized row and column coordinates, respectively. The coordinates range from $(0,0)$ at the top-left corner to $(1,1)$ at the bottom-right corner. The context $x\sim\Unif([0,1]^2)$ specifies the center of a Gaussian reward bump, yielding the mean reward
    \[
        \mu(x,a_e)=0.25+0.04\,d_e
        +0.45\exp\!\left(-\frac{\nbr{(u_e,v_e)-x}_2^2}{2\cdot0.35^2}\right).
    \]
    The mean lies in $[0.25,0.74]$, with higher expected rewards for edges closer to the context location $x$ and a small additional reward for downward edges.

    At each round $t$, a context $x_t$ is drawn from $\Unif([0,1]^2)$, and the edge rewards are drawn independently conditional on $x_t$ as $r_t(e)\mid x_t\sim\mathrm{Bernoulli}(\mu(x_t,a_e))$. The learner receives the sum of the realized rewards on its chosen path and observes the reward of each selected edge as semi-bandit feedback. The regression oracles use the tensor-product features $\phi(x,a_e)=\mathrm{vec}\big((1,a_e)(1,x)^\top\big)\in\RR^{12}$, where $(1,a_e)\in\RR^4$ and $(1,x)\in\RR^3$ prepend a constant to the edge feature and context, respectively.

    \item \textbf{MSLR-WEB30k.}
    Released by Microsoft Research~\citep{qin2013letor}, the whole dataset contains $31{,}531$ queries sampled from Bing's search logs, each associated with a candidate set of documents and a relevance label per (query, document) pair, with an integer relevance scale in $\{0,1,2,3,4\}$. Each (query, document) pair is summarized by $136$ features.
    After filtering out queries with fewer than $10$
    candidate documents, we are left with $30{,}846$ queries and $1{,}489{,}911$ query-document records.

    \item \textbf{Yahoo!\ LTR Set~1.}
    The Yahoo!\ Learning-to-Rank Challenge corpus~\citep{chapelle2011yahoo} provides $29{,}921$ queries sampled from Yahoo!\ search logs, each associated with a candidate set of documents and a relevance label per (query, document) pair, with an integer relevance scale in $\{0,1,2,3,4\}$. Each (query, document) pair is summarized by $519$ features. After filtering out queries with fewer than $6$ candidate documents, we are left with $27{,}630$ queries and $669{,}295$ query-document records.

\end{itemize}

\paragraph{Execution details.}
For each experimental setting and algorithm configuration, we tune the hyperparameter on $10$ runs and benchmark the selected hyperparameter on $10$ additional, independent runs.
Each run lasts $T=30{,}000$ rounds in both the non-contextual $m$-set and contextual grid-path experiments. For the contextual grid-path experiment, all algorithms receive the same context and edge-reward sequences within each seed. For both \sqcomb\ and \textsc{SquareCB.Lin} in this experiment, we normalize feature rows to unit $\ell_2$ norm with \texttt{lin} and use unnormalized features with \texttt{gb2} and \texttt{gb5}.
For the learning-to-rank experiments, each run makes one full pass through the dataset, and it takes $T=30{,}846$ rounds for MSLR-WEB30k and $T=27{,}630$ rounds for Yahoo!\ LTR Set~1.
Following \citet{krishnamurthy2016contextual}, we set the candidate pool size $A=10$ and the slate size $m=3$ for MSLR-WEB30k, and $A=6$ and $m=2$ for Yahoo!\ LTR Set~1.

\paragraph{From a learning-to-rank corpus to a CCSB instance}

We turn each corpus into an online CCSB simulator. Round $t$ proceeds in three steps:
\begin{enumerate}
    \item \textbf{Context.} The next query is pulled from a pool, and its feature vector $x_t$ is revealed to the algorithm as the round's context.
    \item \textbf{Action.} The simulator exposes a candidate set of $A$ documents associated with $x_t$ (the first $A$ documents in stored order when the underlying pool exceeds $A$; as mentioned above, we have already removed queries with fewer than $A$ candidates). The algorithm draws a slate $s_t$ of $m$ documents.
    \item \textbf{Feedback.} The simulator reveals the relevance label of each document in $s_t$ (semi-bandit feedback).
\end{enumerate}

The feedback is read from the offline corpus and is deterministic given the query and selected documents. Randomness comes from the shuffled query order and the learner's randomized choices.

\subsection{Algorithms compared}
\label{sec:exp-algorithms}

The following implementation details include the regression oracle, our proposed algorithm and other baselines used in four experimental settings. The non-contextual $m$-set experiment uses only linear regression oracles since there is no contextual information and the linear model is well-specified. The other three contextual experiments use both linear and gradient-boosted regression oracles, combined with the \sqcomb\ algorithm and all baselines except \textsc{LinUCB}, which is linear-only.

\paragraph{Regression oracles.}
All oracle-based methods use the following regression families:
\begin{itemize}[leftmargin=*,itemsep=1pt,topsep=2pt]
    \item \texttt{lin}: a linear regressor on the $d$-dimensional features, with row-wise $\ell_2$ normalization for the \sqcomb\ and \textsc{SquareCB.Lin}.
    \item \texttt{gb2}, \texttt{gb5}: sklearn's \texttt{GradientBoostingRegressor} with $100$ trees of depth $d \in \{2, 5\}$, refit from scratch on the cumulative interaction buffer at a doubling schedule of rounds.

\end{itemize}

\paragraph{Our method.}
\sqcomb\ is our \Cref{alg:sqcomb}. At each round it solves the log-barrier optimization problem (in step 6)
over marginal probability vectors $\bar{p}_t \in \bar{\Scal}$ and then draws a size-$m$ slate whose per-arm marginals match $\bar{p}_t$ exactly via Madow systematic sampling~\citep{madow1949theory}. The exploration parameter follows the schedule $\gamma_t = \gamma_0 \sqrt{A t / m}$, with $\gamma_0$ being a hyperparameter. We instantiate the regression oracle with the linear and gradient-boosted tree backbones (\texttt{lin}, \texttt{gb2}, \texttt{gb5}).

\paragraph{Baselines.}
We use the following baselines to compare against our method:
\begin{itemize}
    \item \textsc{SquareCB.Lin}~\citep{foster2020adapting} is a contextual bandit algorithm that each round it selects one arm. To apply it to the semi-bandit feedback setting, we treats each combinatorial action/slate $s \in \Scal$ as a single arm, and draws a slate from the $\binom{A}{m}$ size-$m$ combinatorial action set.
    The algorithm draws a slate from a distribution that minimizes a log-determinant barrier objective over the combinatorial action set. Its regression oracle receives one pair per round: the sum of selected document-feature vectors and the total slate reward. While such feature aggregation preserves model well-specification property with linear reward function approximation, we note that it  may not preserved model well-specification for nonlinear function approximation such as decision trees.

    We set the exploration parameter $\gamma_t = \gamma_0\sqrt{At}/m$ for \textsc{SquareCB.Lin}, with $\gamma_0$ being its hyperparameter.
    This choice of $\gamma_t$ is different from \sqcomb, because the online regression oracle that takes into aggregated reward feedback has an $\iupbound{m^2\log\abs{\Fcal}}$ square-loss regret guarantee, whereas the batch-mode oracle in \sqcomb\ has an $\iupbound{m\log\abs{\Fcal}}$ guarantee
    (\Cref{sec:sqlin-full-bandit}).

    \item \textsc{VCEE}~\citep{krishnamurthy2016contextual} is the policy-elimination CCSB algorithm and is, to our knowledge, the strongest published semi-bandit baseline on these corpora.
    \item \textsc{LinUCB}~\citep{chu11contextual} maintains a ridge estimator of the linear reward and a UCB-style confidence radius $\theta^{\!\top} \phi + \alpha \sqrt{\phi^{\!\top}\Sigma^{-1}\phi}$, with $\Sigma$ refreshed every $100$ rounds; it represents the linear-bandit baseline.
    \item \textsc{$\varepsilon$-greedy}~\citep{langford2007epoch} learns a policy from importance-weighted semi-bandit observations from all rounds via an argmax oracle. It plays a uniformly random size-$m$ slate with probability $\varepsilon_t=\max\{t^{-1/2},\varepsilon\}$ and otherwise follows its learned policy.$\varepsilon$.

    \item \textsc{Uniform-random} draws a uniformly random size-$m$ subset every round and never learns.
    \item \textsc{Skyline} is an in-sample supervised learning comparator: a regression oracle (one per backbone family)
    is fit once on the entire labeled training corpus and then, every round, greedily plays the top-$m$ arms under its own predictions. We report one skyline per regression oracle (\texttt{lin}, \texttt{gb2}, \texttt{gb5}) in \Cref{tab:main-final}; the learning-curve figures show the skyline matching each panel's oracle (\texttt{gb5} in \Cref{fig:main}).
\end{itemize}

All oracle-based methods (\sqcomb, \textsc{SquareCB.Lin}, \textsc{VCEE}, $\varepsilon$-greedy, and the supervised \textsc{Skyline} ceiling) are reported with the \texttt{lin}, \texttt{gb2}, and \texttt{gb5} oracles. \textsc{LinUCB} is linear-only. In the non-contextual $m$-set experiment, we report only linear models.
\subsection{Experimental protocol}
\label{sec:exp-protocol}

For all four experimental settings,
we adopt a two-stage protocol: Stage~1 uses seeds $I \in \{0, \ldots, 9\}$ to shuffle the streaming data, for each (algorithm, oracle, dataset) tuple and each $m$ in the non-contextual experiment, the hyperparameter value that yields the highest mean realized cumulative reward, and Stage~2 benchmarks the Stage-1 winners from scratch on the disjoint seeds $I \in \{10, \ldots, 19\}$. Both stages are evaluated at the horizons specified above. We report the mean and sample standard deviation over the $10$ benchmark runs. The realized cumulative reward of a single run is
\[
    \mathrm{cum\_reward}(v, I) \;=\; \sum_{t=1}^{T} \sum_{a: s_t(a) = 1} r_t(a),
\]

where $s_t$ is the action played at round $t$ and $r_t(a)$ is the realized arm reward, in $\{0,1\}$ for the synthetic tasks and $\{0,1,2,3,4\}$ for learning-to-rank. Selecting by realized reward keeps the tuning objective on the same scale as the reported benchmark and avoids the instability of a prediction-based criterion. The hyperparameter grids for learning-to-rank are given in \Cref{tab:appendix-grids}. The synthetic tasks use the same grids, except that \textsc{VCEE} uses $\mu\in\{0.1,0.2,0.5,1,2,5,10,20,50,100\}$ and \textsc{LinUCB} tunes $\delta$ over the grid listed for $\alpha$.

\begin{table}[H]
    \centering
    \caption{Per-algorithm scalar tuning parameters and search grid (same on both learning-to-rank datasets).}
    \label{tab:appendix-grids}
    \footnotesize
    \begin{tabular}{l l l r}
        \toprule
        Algorithm & Hyperparameter & Grid & \# pts \\
        \midrule
        \sqcomb                    & $\gamma_0$    & $\{1, 2, 5\} \cdot 10^{k},\ k = -1, 0, 1$, plus $10^{2}$ & 10 \\
        \textsc{SquareCB.Lin}      & $\gamma_0$    & same as \sqcomb                                        & 10 \\
        \textsc{VCEE}              & $\mu$         & $\{1, 3\} \cdot 10^{k},\ k = -4, \ldots, 0$            & 10 \\
        $\varepsilon$-greedy       & $\varepsilon$ & $\{1, 2, 5\} \cdot 10^{k},\ k = -3, -2, -1$, plus $1$  & 10 \\
        \textsc{LinUCB}            & $\alpha$      & $\{1, 3\} \cdot 10^{k},\ k = -3, \ldots, 1$            & 10 \\
        \bottomrule
    \end{tabular}
\end{table}

\raggedbottom
\subsection{Reward Results}
\label{sec:appendix-curves}

\paragraph{Non-contextual $m$-set results.}
\Cref{tab:mset-regret} reports cumulative pseudo-regret
$\sum_{t=1}^{T}(\max_{s\in\Scal}\langle s,\mu\rangle-\langle s_t,\mu\rangle)$.
\Cref{fig:mset-curves} shows per-round average realized reward, matching the metric in the other learning-curve figures.
\textsc{LinUCB} has the lowest final mean pseudo-regret for $m=1,\ldots,7$ and $m=9$; \sqcomb\ with \texttt{lin} leads at $m=8$. \sqcomb\ has lower final mean pseudo-regret than other baselines except \textsc{LinUCB} for $m=3,\ldots,9$. At $m=10$, every method plays the unique feasible action and has zero pseudo-regret.

\begin{table}[t]
    \centering
    \footnotesize
    \setlength{\tabcolsep}{3pt}
    \caption{
    Cumulative pseudo-regret in the non-contextual $m$-set experiment at $T=30{,}000$. Entries are means over $10$ runs, with sample standard deviations in parentheses. The lowest mean in each column is bold.}
    \label{tab:mset-regret}
    \resizebox{\linewidth}{!}{
    \begin{tabular}{l ccccc}
\toprule
Algorithm & $m=1$ & $m=2$ & $m=3$ & $m=4$ & $m=5$ \\
\midrule
\sqcomb{} (\texttt{lin}) & $263.2\,(199.1)$ & $778.7\,(917.8)$ & $356.8\,(308.9)$ & $307.5\,(119.7)$ & $417.6\,(211.3)$ \\
\textsc{SquareCB.Lin} (\texttt{lin}) & $227.8\,(113.2)$ & $416.9\,(66.7)$ & $500.7\,(84.7)$ & $888.0\,(367.7)$ & $1273.5\,(522.3)$ \\
\textsc{VCEE} & $1520.1\,(971.5)$ & $870.3\,(619.8)$ & $598.6\,(190.4)$ & $788.8\,(303.6)$ & $732.8\,(243.9)$ \\
$\varepsilon$-greedy & $833.9\,(513.5)$ & $507.5\,(340.6)$ & $439.1\,(101.4)$ & $790.8\,(246.1)$ & $995.1\,(592.4)$ \\
\textsc{LinUCB} & $\mathbf{49.0\,(10.0)}$ & $\mathbf{50.8\,(11.8)}$ & $\mathbf{39.0\,(9.1)}$ & $\mathbf{182.0\,(383.5)}$ & $\mathbf{198.0\,(377.8)}$ \\
\midrule
Algorithm & $m=6$ & $m=7$ & $m=8$ & $m=9$ & $m=10$ \\
\midrule
\sqcomb{} (\texttt{lin}) & $323.0\,(201.8)$ & $203.4\,(65.9)$ & $\mathbf{186.8\,(106.3)}$ & $92.5\,(79.6)$ & $\mathbf{0.0\,(0.0)}$ \\
\textsc{SquareCB.Lin} (\texttt{lin}) & $1464.7\,(637.8)$ & $1600.7\,(667.3)$ & $1925.3\,(289.2)$ & $1060.4\,(790.5)$ & $\mathbf{0.0\,(0.0)}$ \\
\textsc{VCEE} & $670.1\,(360.8)$ & $468.8\,(270.2)$ & $414.1\,(234.9)$ & $134.6\,(77.9)$ & $\mathbf{0.0\,(0.0)}$ \\
$\varepsilon$-greedy & $924.9\,(482.5)$ & $501.0\,(82.9)$ & $279.9\,(125.7)$ & $329.1\,(290.6)$ & $\mathbf{0.0\,(0.0)}$ \\
\textsc{LinUCB} & $\mathbf{48.9\,(19.8)}$ & $\mathbf{49.7\,(30.0)}$ & $289.5\,(770.9)$ & $\mathbf{29.9\,(29.8)}$ & $\mathbf{0.0\,(0.0)}$ \\
\bottomrule
\end{tabular}
    }
\end{table}

\begin{figure}[t]
    \centering
    \includegraphics[width=\linewidth]{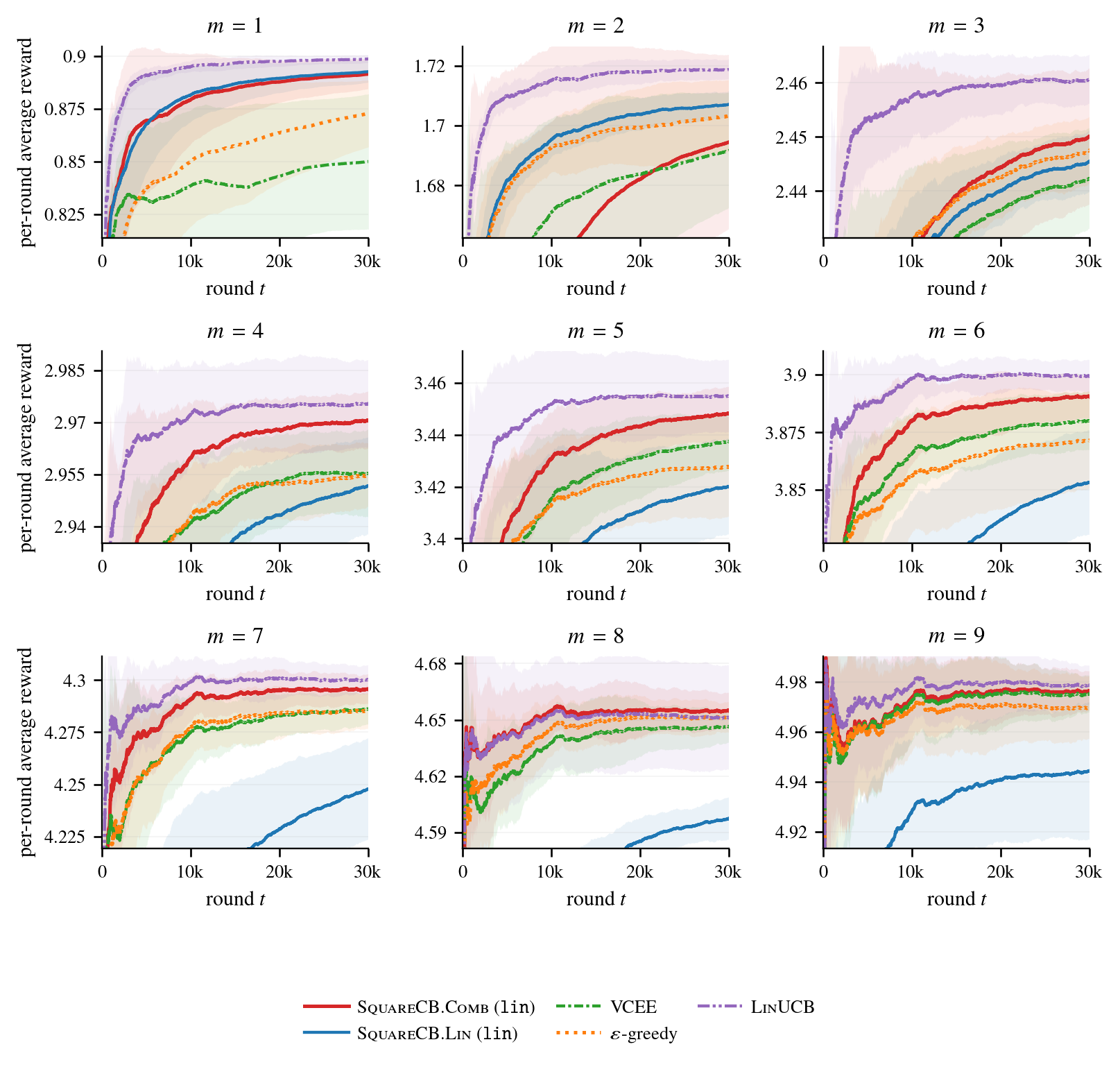}
    \caption{Per-round average realized reward in the non-contextual $m$-set experiment over $30{,}000$ rounds for $m=1,\ldots,9$. Curves show means over $10$ runs, with $\pm 1$ sample standard deviation.
    }
    \label{fig:mset-curves}
\end{figure}

\paragraph{Contextual grid-path results.}
\Cref{tab:grid-path-reward} reports realized reward per round at $T=30{,}000$, and the top row of \Cref{fig:appendix-curves} shows the learning curves. Among online learners, \sqcomb\ leads with both tree oracles, reaching $3.9151\pm0.0157$ with \texttt{gb2} and $3.9226\pm0.0117$ with \texttt{gb5}. The corresponding \textsc{VCEE} means are $3.8466$ and $3.8397$, while $\varepsilon$-greedy achieves $3.8073$ and $3.8145$. Both baselines improve over their linear versions, but remain below \sqcomb\ with the same tree oracle. \textsc{SquareCB.Lin} has the highest final mean with \texttt{lin}. The fitted \textsc{Skyline} references use each oracle on all realized rewards of the evaluation stream.

It is worth noting that \textsc{SquareCB.Lin} performs poorly with the tree oracles. We speculate that it is because we aggregate the features of the $m$ selected arms into a single vector and feed it to the regression oracle, which is compatible with linear models but incompatible with nonlinear models. In the non-contextual $m$-set experiments, \textsc{SquareCB.Lin} performs relatively well because the features are the indicator vectors of the arms and the aggregation of rewards is linear.

\begin{table}[H]
    \centering
    \footnotesize
    \caption{Realized reward per round on the contextual grid-path instance at $T=30{,}000$: means over $10$ runs, with sample standard deviations in parentheses. The highest online mean within each oracle column is bold. \textsc{VCEE} and $\varepsilon$-greedy use path-constrained policies with each indicated oracle; \textsc{LinUCB} is linear-only (``---'' otherwise). \textsc{Skyline} fits the indicated oracle on all realized rewards in the evaluation stream.
    }
    \label{tab:grid-path-reward}
    \begin{tabular}{l ccc}
\toprule
Algorithm & \texttt{lin} & \texttt{gb2} & \texttt{gb5} \\
\midrule
\textsc{Skyline} & $3.6071\,(0.0081)$ & $3.9446\,(0.0104)$ & $4.0086\,(0.0075)$ \\
\midrule
\sqcomb{} & $3.8338\,(0.0302)$ & $\mathbf{3.9151\,(0.0157)}$ & $\mathbf{3.9226\,(0.0117)}$ \\
\textsc{SquareCB.Lin} & $\mathbf{3.8449\,(0.0168)}$ & $3.4146\,(0.0178)$ & $3.4046\,(0.0571)$ \\
\textsc{VCEE} & $3.5641\,(0.0226)$ & $3.8466\,(0.0105)$ & $3.8397\,(0.0118)$ \\
$\varepsilon$-greedy & $3.5828\,(0.0069)$ & $3.8073\,(0.0177)$ & $3.8145\,(0.0120)$ \\
\textsc{LinUCB} & $3.6060\,(0.0080)$ & --- & --- \\
\bottomrule
\end{tabular}
\end{table}

\paragraph{Contextual slate recommendation results.}
In \Cref{fig:appendix-curves}, we compare the all semi-bandit algorithms across the three regression oracles reported in \Cref{tab:main-final} (\texttt{lin}, \texttt{gb2}, \texttt{gb5}); \textsc{LinUCB} exclusively uses linear regression and therefore appears only in the \texttt{lin} column. On the learning-to-rank tasks, \sqcomb\ has the highest mean reward among all algorithms with \texttt{gb2} and \texttt{gb5}, while \textsc{LinUCB} leads with \texttt{lin}. \sqcomb\ outperforms \textsc{SquareCB.Lin} in all six settings. This demonstrates that the additional complexity of richer tree oracles do not consistently improve \textsc{SquareCB.Lin}. Policy-based $\varepsilon$-greedy follows \textsc{VCEE} in every configuration.
\Cref{tab:main-final} reports the per-round average reward at the final round $t=T$ as the mean $\pm$ sample standard deviation over the $10$ runs.

\begin{table}[t]
    \centering
    \footnotesize
    \setlength{\tabcolsep}{3pt}
    \caption{Average per-round reward at the final round, broken out by regression oracle (\texttt{lin}, \texttt{gb2}, \texttt{gb5}). Mean $\pm$ sample standard deviation over $10$ runs. In each column, the highest average reward among the algorithms (excluding the \textsc{Skyline}) is shown in bold. \textsc{Skyline} and \textsc{Uniform-random} are non-online learning references that upper and lower bound the performance of online algorithms; \textsc{Skyline} is reported separately for each regression oracle.
    \textsc{LinUCB} is linear-only.}
    \label{tab:main-final}
    \resizebox{\linewidth}{!}{
    \begin{tabular}{l ccc ccc}
        \toprule
        & \multicolumn{3}{c}{MSLR-WEB30k} & \multicolumn{3}{c}{Yahoo!\ LTR Set~1} \\
        \cmidrule(lr){2-4}\cmidrule(lr){5-7}
        Algorithm & \texttt{lin} & \texttt{gb2} & \texttt{gb5} & \texttt{lin} & \texttt{gb2} & \texttt{gb5} \\
        \midrule
        \textsc{Skyline} & $2.340 \pm 0.003$ & $2.402 \pm 0.003$ & $2.536 \pm 0.003$ & $3.209 \pm 0.000$ & $3.219 \pm 0.000$ & $3.335 \pm 0.000$ \\
        \midrule
        \textbf{\sqcomb}\ & $2.022 \pm 0.006$ & $\mathbf{2.373 \pm 0.006}$ & $\mathbf{2.404 \pm 0.004}$ & $3.140 \pm 0.003$ & $\mathbf{3.186 \pm 0.002}$ & $\mathbf{3.207 \pm 0.005}$ \\
        \textsc{SquareCB.Lin} & $1.990 \pm 0.005$ & $2.141 \pm 0.044$ & $2.024 \pm 0.050$ & $3.125 \pm 0.003$ & $3.105 \pm 0.008$ & $3.053 \pm 0.008$ \\
        \textsc{VCEE} & $2.252 \pm 0.007$ & $2.309 \pm 0.009$ & $2.366 \pm 0.010$ & $3.079 \pm 0.029$ & $3.129 \pm 0.005$ & $3.182 \pm 0.008$ \\
        $\varepsilon$-greedy & $2.228 \pm 0.006$ & $2.284 \pm 0.008$ & $2.303 \pm 0.010$ & $3.040 \pm 0.007$ & $3.118 \pm 0.006$ & $3.154 \pm 0.008$ \\
        \textsc{LinUCB} & $\mathbf{2.330 \pm 0.005}$ & --- & --- & $\mathbf{3.166 \pm 0.003}$ & --- & --- \\
        \midrule
        \textsc{Uniform-random} & \multicolumn{3}{c}{$1.815 \pm 0.005$} & \multicolumn{3}{c}{$2.592 \pm 0.006$} \\
        \bottomrule
    \end{tabular}
    }
\end{table}

\begin{figure}[t]
    \centering
    \includegraphics[width=\linewidth]{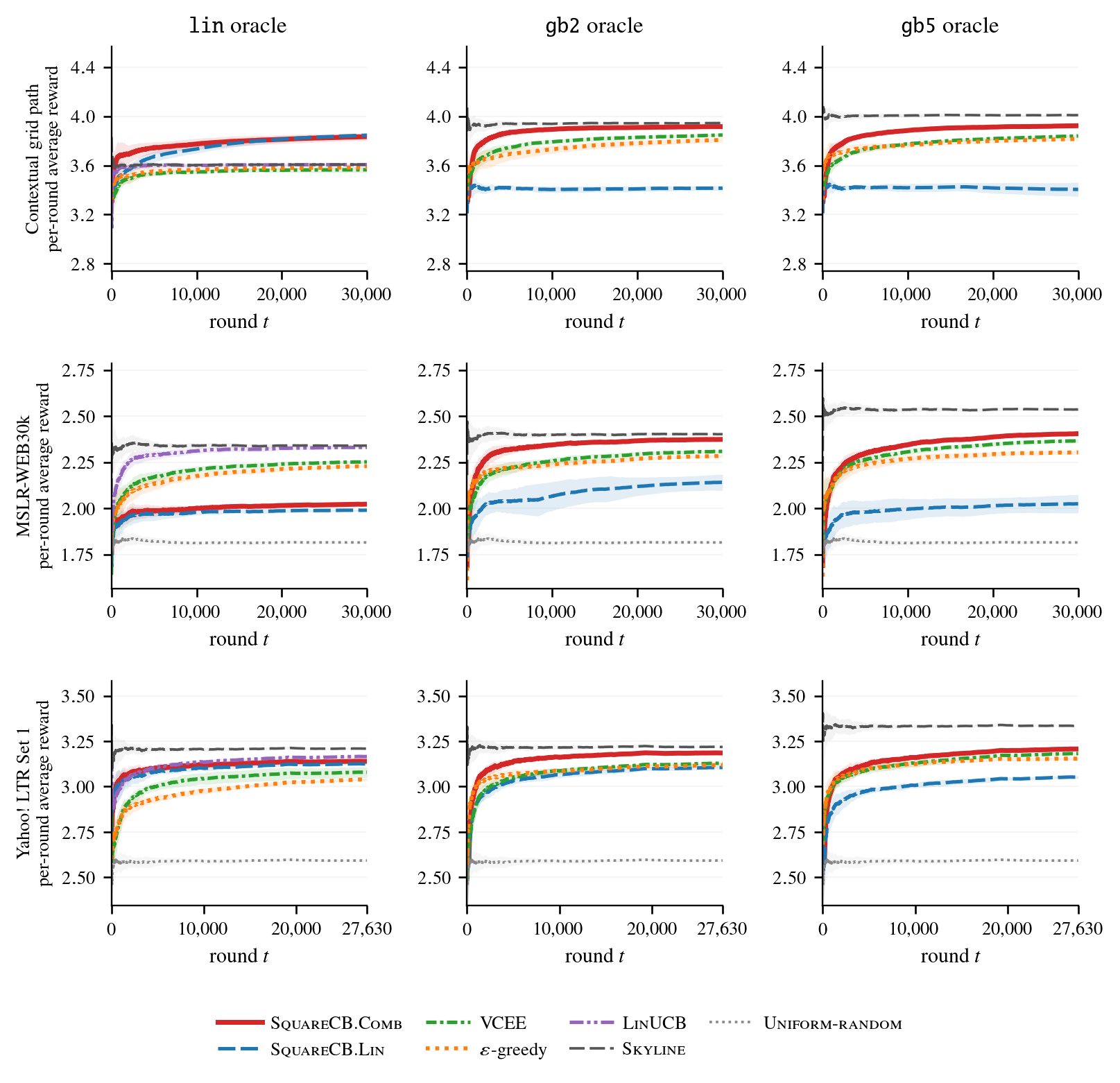}
    \caption{Per-round average reward by regression oracle: rows are the contextual grid-path experiment (top), MSLR-WEB30k (middle), and Yahoo!\ LTR Set~1 (bottom) and columns are the \texttt{lin}, \texttt{gb2}, and \texttt{gb5} oracles. \textsc{LinUCB} appears only in the \texttt{lin} column. \textsc{Skyline} is fitted separately for each oracle. \textsc{Uniform-random} is shown only for the contextual slate recommendation tasks. Curves are the means over $10$ runs; shaded bands show $\pm 1$ sample standard deviation across those runs. The bottom two panels of the \texttt{gb5} column reproduce \Cref{fig:main}, except for \textsc{LinUCB}.}
    \label{fig:appendix-curves}
\end{figure}

\subsection{Running Time Results for Grid-path Experiment}
\label{sec:exp-grid-runtime}
We run the grid-path experiments on a machine with an AMD Ryzen~9~5950X ($16$ cores, $64$ GB RAM), using single-threaded workers. We report the computational cost per run at $T=5000$ in \Cref{tab:grid-path-cost}, averaged over $10$ runs, separating total wall-clock time from time spent inside the regression oracle.

On this instance, \sqcomb\ is faster than \textsc{SquareCB.Lin} and \textsc{VCEE} under both tree-based oracles.
With \texttt{gb5}, \sqcomb\ takes $38.33 \pm 6.26$s per run, compared with $3268.61 \pm 354.30$s for \textsc{VCEE} and $279.27 \pm 17.53$s for \textsc{SquareCB.Lin}, making it about $7.3$ times faster than the latter. $\varepsilon$-greedy is the fastest learning algorithm, and achieves higher mean reward than \sqcomb\ with the tree oracles and lower mean reward with \texttt{lin} in this $T=5000$ benchmark.

The additional running time of \textsc{SquareCB.Lin} is primarily outside regression, consistent with the cost of its log-determinant sampling step. \textsc{VCEE} solves an optimization problem over the policy space through repeated calls to an argmax oracle (AMO). These calls are included in total time, but not in the regression-oracle time.

\begin{table}[H]
    \centering
    \footnotesize
    \setlength{\tabcolsep}{3pt}
    \caption{Computational cost per run in the earlier fixed-parameter grid-path benchmark at $T=5000$ (mean over $10$ runs, sample standard deviation in parentheses): total wall-clock time and, within it, time spent inside the regression oracle. Columns correspond to regression oracles; ``---'' denotes combinations that were not meaningful. \textsc{LinUCB} and \textsc{Random} call no external regression oracle, and \textsc{LinUCB}'s regression uses inline linear algebra, so their regression-oracle time is zero by construction. Calls to the argmax oracle (AMO), used by \textsc{VCEE} and $\varepsilon$-greedy, are included in total time but not in the regression-oracle columns.}
    \label{tab:grid-path-cost}
    \resizebox{\linewidth}{!}{
    \begin{tabular}{l rrr rrr}
        \toprule
        & \multicolumn{3}{c}{Total (s)} & \multicolumn{3}{c}{Regression oracle (s)} \\
        \cmidrule(lr){2-4}\cmidrule(lr){5-7}
        Algorithm & \texttt{lin} & \texttt{gb2} & \texttt{gb5} & \texttt{lin} & \texttt{gb2} & \texttt{gb5} \\
        \midrule
        \textbf{\sqcomb}\ & $15.21\,(0.12)$ & $25.76\,(4.94)$ & $38.33\,(6.26)$ & $2.29\,(0.06)$ & $12.23\,(2.56)$ & $24.44\,(3.63)$ \\
        \textsc{SquareCB.Lin} & $254.88\,(1.75)$ & $287.19\,(55.30)$ & $279.27\,(17.53)$ & $4.83\,(0.15)$ & $5.01\,(0.88)$ & $6.78\,(0.76)$ \\
        \textsc{VCEE} & $28.18\,(3.45)$ & $1908.08\,(159.20)$ & $3268.61\,(354.30)$ & $2.50\,(0.06)$ & $17.01\,(2.14)$ & $20.33\,(1.45)$ \\
        $\varepsilon$-greedy & $1.13\,(0.04)$ & $12.10\,(0.47)$ & $26.54\,(0.72)$ & $0.25\,(0.00)$ & $0.74\,(0.07)$ & $0.88\,(0.04)$ \\
        \textsc{LinUCB} & $1.45\,(0.03)$ & --- & --- & $0.00\,(0.00)$ & --- & --- \\
        \midrule
        \textsc{Random} & \multicolumn{3}{c}{$0.45\,(0.01)$} & \multicolumn{3}{c}{$0.00\,(0.00)$} \\
        \bottomrule
    \end{tabular}
    }
\end{table}

%% file: Contents/9.5.lower-bound.tex
\section{Lower Bound for the CCSB Problem}
\label{sec:appendix-lower-bound}

In this section, we establish a lower bound on the CCSB regret in terms of $\abs{\Fcal}$, $T$, $m$, and $A$. Our lower bound construction uses a contextual $m$-path problem with $A$ arms and a function class $\Fcal$ of size at most $N$, building on the non-contextual $m$-path lower bound of \citet{kveton2015tight} and the realizable contextual bandit lower bound of~\citet{agarwal12contextual}.
In the contextual $m$-path problem, the learner chooses an $m$-edge path at each round, and each $f \in \Fcal$ maps from the space of (context, arm) pairs to $[0,1]$. We partition the time horizon $T$ into $M$ non-overlapping intervals of equal length $T/M$.

The main result of this section is the following lower bound for CCSB with general function approximation.

\lowerBoundCCSB*

\begin{proof}[Proof of \Cref{thm:lower-bound-ccsb}] We will first give the construction of a family of contextual shortest path instances and argue that $\Alg$ must suffer a large regret in one of the instances.

\paragraph{The construction.}
Let $M = \lfloor \log_{A/m} N \rfloor$ and $\tau = T / M$; both are positive integers under our assumptions that $N \geq A/m$ and $T/M \in \NN$. Define the context space $\Xcal = \cbr{1, \ldots, M}$ and the combinatorial action space
\[
    \Scal = \cbr{ s_j = \one_{p_j}: j \in [A/m]}, \qquad p_j = \cbr{m(j-1) + 1, \ldots, m j},
\]
We use $s_j$ to represent the $j$-th path from the start to the goal; see Figure~\ref{fig:m-path} for an illustration

\begin{figure}[t]
    \centering
    \includegraphics[width=\linewidth]{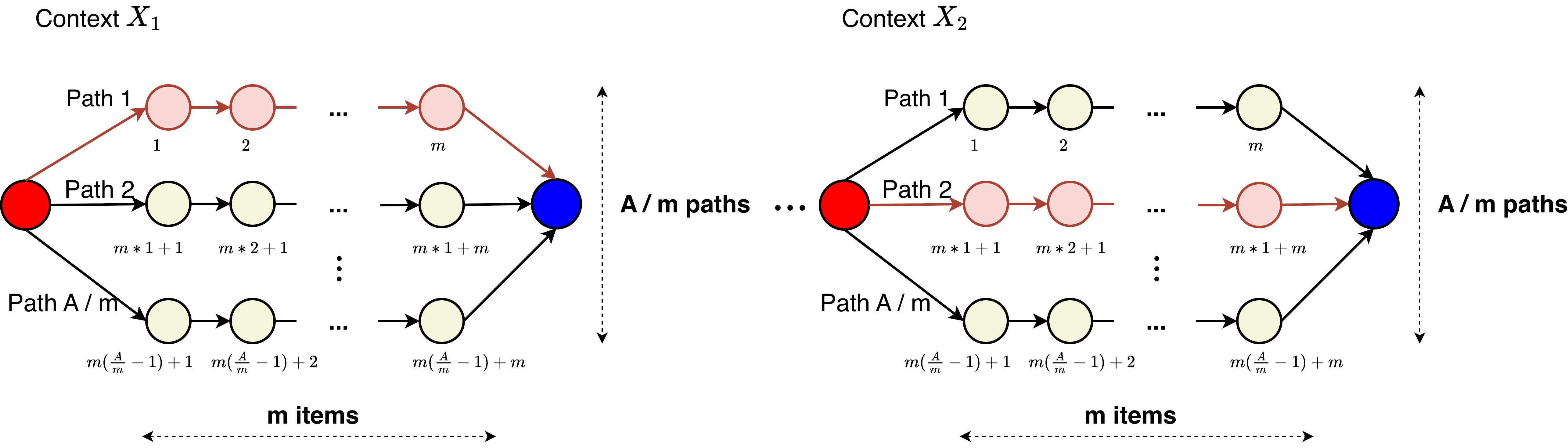}
    \caption{Contextual $m$-path construction. Each $m \times (A/m)$ grid depicts the shared arm set $[A]$, with cell $(j, k)$ being the $k$-th node of the $j$-th path. The optimal path is shaded pink and yields expected reward $1/2$; the remaining $A/m - 1$ paths are shaded beige and yield reward $1/2 - \Delta$. Different contexts induce different optimal paths, so the reward distributions across contexts are independent.}
    \label{fig:m-path}
\end{figure}

Each combinatorial action has cardinality exactly $m$, so the size constraint of the CCSB problem holds.

For each $j \in [A/m]$, define $g_j : \Acal \to [0,1]$ by
\[
    g_j(a) =
    \begin{cases}
        \tfrac{1}{2}, & a \in p_j, \\
        \tfrac{1}{2} - \Delta, & a \notin p_j,
    \end{cases}
    \qquad \text{where } \Delta = \sqrt{\frac{A}{m\,\tau}} \in (0, \tfrac{1}{4}].
\]
Our assumption $\tau \geq 16 A / m$ ensures that $\Delta \leq 1/4$. Set $\Gcal = \cbr{g_j : j \in [A/m]}$ and
\[
    \Fcal = \cbr{f : f(i, \cdot) \in \Gcal \text{ for every } i \in [M]}.
\]
By construction, $\abs{\Fcal} = (A/m)^M \leq N$.

For each $f \in \Fcal$, define the environment $E_f$ as follows: for $t$ in the $i$-th interval $[(i-1)\tau + 1, i\tau]$, the environment shows context $x_t = i$, and the reward vector $r_t$ is drawn by first sampling path rewards $w_{t,1}, \ldots, w_{t,A/m}$ independently with $w_{t,j} \sim \mathrm{Bernoulli}\bigl(f(i, m(j-1) + 1)\bigr)$ and then setting $r_t(a) = w_{t,j}$ for all $a \in p_j$. In other words, the \textit{realized} rewards of the arms within each path are identical; see \Cref{fig:m-path}. By construction $\EE\sbr{r_t(a) \mid x_t} = f(x_t, a)$, so $E_f$ is realizable with respect to $\Fcal$.

\textbf{Lower bound from the non-contextual $m$-path problem.}
For any fixed $g \in \Gcal$, define $E_g$ to be the environment with horizon $\tau$ in which $r_t$ is drawn by sampling $w_{t,1}, \ldots, w_{t,A/m}$ independently with $w_{t,j} \sim \mathrm{Bernoulli}\bigl(g(m(j-1) + 1)\bigr)$ and setting $r_t(a) = w_{t,j}$ for all $a \in p_j$. We will use the following non-contextual lower bound.

\begin{lemma}[\citealp{kveton2015tight}, Proposition~2]
\label{lemma:lower-bound-m-path}
For any $m, A, \tau \in \NN$ with $A/m \in \NN$ and $\tau \geq 16 A / m$, and any algorithm $\Alg$,
\[
    \EE_{g^\star \sim \Unif(\Gcal)} \EE_{\Alg, E_{g^\star}} \sbr{ \sum_{t=1}^\tau \max_{s \in \Scal}\inner{g^\star}{s} - \inner{g^\star}{s_t} }
    \geq \ilowbound{\sqrt{A m \tau}}.
\]
\end{lemma}

\textbf{Aggregating across intervals.}
Let $\Alg$ be any CCSB algorithm and let $p_t \in \Delta(\Scal)$ denote its action distribution at round $t$. The expected regret of $\Alg$ on environment $E_{f^\star}$ is
\[
    \Regret_{f^\star}(\Alg) = \sum_{t=1}^T \max_{s \in \Scal}\inner{f^\star(x_t, \cdot)}{s} - \EE_{s \sim p_t}\inner{f^\star(x_t, \cdot)}{s}.
\]
Sampling $f^\star$ uniformly at random from $\Fcal$ gives a Bayesian lower bound on the worst case:
\[
    \sup_{f^\star \in \Fcal} \Regret_{f^\star}(\Alg)
    \;\geq\; \EE_{f^\star \sim \Unif(\Fcal)} \Regret_{f^\star}(\Alg).
\]
Splitting the sum across the $M$ intervals,
\begin{align*}
    \EE_{f^\star \sim \Unif(\Fcal)} \Regret_{f^\star}(\Alg)
    =
    \sum_{i=1}^M \EE_{f^\star \sim \Unif(\Fcal)}
    \sbr{
    \sum_{t=(i-1)\tau + 1}^{i\tau}\!\! \max_{s \in \Scal}\inner{f^\star(i, \cdot)}{s} - \EE_{s \sim p_t}\inner{f^\star(i, \cdot)}{s}
    }.
\end{align*}
Since $\Fcal = \Gcal^M$ as a Cartesian product, drawing $f^\star \sim \Unif(\Fcal)$ is equivalent to drawing $f^\star(i, \cdot) \sim \Unif(\Gcal)$ independently for each $i \in [M]$. The history available to $\Alg$ during interval $i$ depends only on $\cbr{f^\star(j, \cdot)}_{j < i}$, which are independent of $f^\star(i, \cdot)$. Hence within interval $i$, $\Alg$ effectively faces a non-contextual $m$-path instance of horizon $\tau$ with reward function drawn from $\Unif(\Gcal)$. \Cref{lemma:lower-bound-m-path} therefore yields
\[
    \EE_{f^\star \sim \Unif(\Fcal)}
    \sbr{
    \sum_{t=(i-1)\tau + 1}^{i\tau}\!\! \max_{s \in \Scal}\inner{f^\star(i, \cdot)}{s} - \EE_{s \sim p_t}\inner{f^\star(i, \cdot)}{s} }
    \;\geq\; \ilowbound{\sqrt{A m \tau}}.
\]
Summing over $i$ and using $\tau = T/M$ with $M = \lfloor \log_{A/m} N \rfloor$,
\[
    \sup_{f^\star \in \Fcal} \Regret_{f^\star}(\Alg)
    \;\geq\; M \cdot \ilowbound{\sqrt{A m T / M}}
    = \ilowbound{\sqrt{A m T M}}
    = \ilowboundlog{\sqrt{A m T \log N}},
\]
where the last step uses $M = \log N / \log(A/m)$ and absorbs the $\log(A/m)$ factor into $\widetilde{\Omega}(\cdot)$.

\end{proof}

%% file: Contents/9.x.supportings.tex
\section{Supporting Lemmas}
\subsection{Off-policy evaluation lemma}

The lemma in this subsection is stated and analyzed at a fixed context $x \in \Xcal$, so we omit the dependence on $x$ for brevity: we write $g(a)$, $g^\star(a)$, and $\Gcal$ to denote $f(x, a)$, $f^\star(x, a)$, and the per-context reward function class $\Fcal_x \coloneqq \cbr{f(x, \cdot) : f \in \Fcal} \subseteq \RR^A$, respectively.

\begin{lemma}
    \label{lemma:ope-error}
    Let $g^\star \in \Gcal$ be the ground truth reward function. For any $g \in \Gcal$ and any marginal probability vectors $\bar p, \bar q \in \bar\Scal$ with $\bar p(a) > 0$ for all $a \in \Acal$, we have
    \begin{align*}
        \abs{\inner{g}{\bar q} - \inner{g^\star}{\bar q}}
        \leq
        \sqrt{ \coverage(\bar p, \bar q) \cdot \opnorm{g - g^\star}{\bar p}^2 }.
    \end{align*}
\end{lemma}

\begin{proof}[Proof of Lemma~\ref{lemma:ope-error}]
    Starting from the squared left-hand side,
    \begin{align*}
        & \del{ \inner{g}{\bar q} - \inner{g^\star}{\bar q} }^2
        = \del{ \sum_{a \in \Acal} \bar q(a) \del{g(a) - g^\star(a)} }^2
        \\
        =&
        \del{ \sum_{a \in \Acal} \frac{\bar q(a)}{\sqrt{\bar p(a)}} \sqrt{\bar p(a)}\del{g(a) - g^\star(a)} }^2
        \\
        \leq&
        \del{\sum_{a \in \Acal} \frac{\bar q(a)^2}{\bar p(a)}}
        \cdot
        \sum_{a \in \Acal}\bar p(a)\del{g(a) - g^\star(a)}^2
            \tag{Cauchy-Schwarz inequality}
        \\
        =&
        \coverage(\bar p, \bar q) \cdot \opnorm{g - g^\star}{\bar p}^2.
    \end{align*}
    Taking a square root on both sides yields the lemma.
\end{proof}

In typical applications, we instantiate $g = \hat f_m(x, \cdot)$ and $g^\star = f^\star(x, \cdot)$ at the current context $x$, with $\bar p = \bar\pi_{m-1}(x)$ denoting the marginal probability vector of the data-collection policy at epoch $m-1$ and $\bar q$ denoting a target marginal probability vector. Then $\opnorm{g - g^\star}{\bar p}^2$ is the on-policy model estimation error, which is controlled by the regression oracle, and $\coverage(\bar p, \bar q)$ measures how well $\bar p$ covers $\bar q$. The lemma thus states that the off-policy evaluation error of any target policy $\bar q$ is bounded by the on-policy estimation error and the coverage of $\bar p$ over $\bar q$.

\subsection{Batch Mode Offline Regression Oracle Guarantees}
In this section, we provide on-policy reward estimation error guarantees for the ERM estimator under full-bandit and semi-bandit feedback, where
\Cref{lemma:off-reg-full-bandit} establishes the former, while \Cref{lemma:off-reg-semi-bandit} establishes the latter.
Recall from \Cref{sec:prelim} that $r_i \in [0,1]^A$ and $\Fcal \subseteq [0,1]^{\Xcal \times \Acal}$.  The full-bandit guarantee is worse than the semi-bandit guarantee by a factor of $m$ (the maximum size of a combinatorial action), which is expected since full-bandit feedback is less informative than semi-bandit feedback.

\begin{lemma}[On-policy model estimation error in full-bandit feedback]
    \label{lemma:off-reg-full-bandit}
    Let $\Dcal_\Xcal$ be the marginal distribution over contexts. Suppose we collect a dataset $\Dcal_n = \cbr{(x_i, s_i, \inner{r_i}{s_i})}_{i=1}^n$ of $n$ i.i.d.\ slate-level observations, where $x_i \sim \Dcal_\Xcal$, $s_i \sim \pi(\cdot \mid x_i)$, and the (hidden) reward vector $r_i \in [0, 1]^A$ satisfies $\EE\sbr{r_i(a) \mid x_i} = f^\star(x_i, a)$ for every $a \in \Acal$. Let $\hat f \in \argmin_{f \in \Fcal} \sum_{i=1}^n \del{\inner{f(x_i, \cdot)}{s_i} - \inner{r_i}{s_i}}^2$ be the ERM estimator. Then for any $\delta \in (0, 1)$, with probability at least $1 - \delta$,
    \begin{align*}
        \EE_{x \sim \Dcal_\Xcal}\sbr{
            \EE_{s \sim \pi(\cdot \mid x)}\sbr{
                \del{\inner{\hat f(x, \cdot) - f^\star(x, \cdot)}{s}}^2
            }
        }
        \lesssim
        \frac{m^2}{n} \log\del{\frac{\abs{\Fcal}}{\delta}}.
    \end{align*}
\end{lemma}
\begin{proof}[Proof of \Cref{lemma:off-reg-full-bandit}]
    For any $f \in \Fcal$, define the per-round excess loss
    \[
        Z_i(f)
        :=
        \del{\inner{f(x_i, \cdot)}{s_i} - \inner{r_i}{s_i}}^2
        -
        \del{\inner{f^\star(x_i, \cdot)}{s_i} - \inner{r_i}{s_i}}^2.
    \]
    Since $\EE\sbr{\inner{r_i}{s_i} \mid x_i, s_i} = \inner{f^\star(x_i, \cdot)}{s_i}$,
    \[
        \EE\sbr{Z_i(f) \mid x_i, s_i}
        =
        \del{\inner{f(x_i, \cdot) - f^\star(x_i, \cdot)}{s_i}}^2,
    \]
    and taking expectation over $(x_i, s_i)$,
    \begin{align*}
        \EE\sbr{Z_i(f)}
        =
        \EE_{x \sim \Dcal_\Xcal}\sbr{
            \EE_{s \sim \pi(\cdot \mid x)}\sbr{
                \del{\inner{f(x, \cdot) - f^\star(x, \cdot)}{s}}^2
            }
        }.
    \end{align*}

    Since $f, f^\star \in [0, 1]^{\Xcal \times \Acal}$, $r_i \in [0, 1]^A$, and $\norm{s_i}_1 \leq m$, we have $\abs{\inner{f(x_i, \cdot)}{s_i}} \leq m$ and $\abs{\inner{r_i}{s_i}} \leq m$, so $\abs{Z_i(f)} \leq m^2$.

    Factoring the excess loss as a difference of squares,
    \[
        Z_i(f)
        =
        \inner{f(x_i, \cdot) - f^\star(x_i, \cdot)}{s_i}
        \cdot
        \del{\inner{f(x_i, \cdot) + f^\star(x_i, \cdot)}{s_i} - 2 \inner{r_i}{s_i}},
    \]
    where the second factor is bounded by $2m$ in absolute value, so
    \[
        Z_i(f)^2
        \leq
        4 m^2 \del{\inner{f(x_i, \cdot) - f^\star(x_i, \cdot)}{s_i}}^2,
    \]
    and taking expectation gives $\EE\sbr{Z_i(f)^2} \leq 4 m^2\, \EE\sbr{Z_i(f)}$.

    Applying Bernstein's inequality and a union bound over $f \in \Fcal$, with probability at least $1 - \delta$, simultaneously for all $f \in \Fcal$,
    \[
        \EE\sbr{Z_1(f)} - \frac{1}{n} \sum_{i=1}^n Z_i(f)
        \lesssim
        m \sqrt{\frac{\EE\sbr{Z_i(f)} \log(\abs{\Fcal}/\delta)}{n}}
        +
        \frac{m^2}{n} \log\del{\frac{\abs{\Fcal}}{\delta}}.
    \]
    Substituting $f = \hat f$ and using $\frac{1}{n} \sum_{i=1}^n Z_i(\hat f) \leq 0$ (by the ERM property), the elementary implication $x \leq c\sqrt{xA} + B \Rightarrow x \lesssim c^2 A + B$ for $A, B \geq 0$ yields
    \[
        \EE\sbr{Z_1(\hat f)}
        \lesssim
        \frac{m^2}{n} \log\del{\frac{\abs{\Fcal}}{\delta}},
    \]
    which is the desired bound by the expression for $\EE\sbr{Z_1(\hat f)}$ derived above.
\end{proof}

\begin{lemma}[On-policy model estimation error in semi-bandit feedback]
    \label{lemma:off-reg-semi-bandit}
    Let $\Dcal_\Xcal$ be the marginal distribution over contexts. Suppose we collect a dataset $\Dcal_n = \cbr{(x_i, s_i, o_i)}_{i=1}^n$ of $n$ i.i.d.\ semi-bandit observations, where $x_i \sim \Dcal_\Xcal$, $s_i \sim \pi(\cdot \mid x_i)$, $o_i = \cbr{o_i(a) : s_i(a) = 1}$ with $o_i(a) = r_i(a)$, and the (hidden) reward vector $r_i \in [0, 1]^A$ satisfies $\EE\sbr{r_i(a) \mid x_i} = f^\star(x_i, a)$ for every $a \in \Acal$. Let $\hat f \in \argmin_{f \in \Fcal} \sum_{i=1}^n \sum_{a: s_i(a) = 1} \del{f(x_i, a) - o_i(a)}^2$ be the ERM estimator. Then for any $\delta \in (0, 1)$, with probability at least $1 - \delta$,
    \begin{align*}
        \EE_{x \sim \Dcal_\Xcal}\sbr{ \opnorm{\hat f - f^\star}{\bar\pi(x)}^2 }
        \lesssim
        \frac{m}{n} \log\del{\frac{\abs{\Fcal}}{\delta}}.
    \end{align*}
\end{lemma}

\begin{proof}[Proof of \Cref{lemma:off-reg-semi-bandit}]

    For any $f \in \Fcal$, define the per-round semi-bandit excess loss
    \[
        Z_i(f)
        :=
        \sum_{a: s_i(a) = 1}
        \sbr{
            \del{f(x_i, a) - o_i(a)}^2
            -
            \del{f^\star(x_i, a) - o_i(a)}^2
        }.
    \]
    Since $\EE\sbr{o_i(a) \mid x_i, s_i} = f^\star(x_i, a)$ whenever $s_i(a) = 1$, the expected excess loss conditioned on $(x_i, s_i)$ is
    \[
        \EE\sbr{Z_i(f) \mid x_i, s_i}
        =
        \sum_{a: s_i(a) = 1}
        \del{f(x_i, a) - f^\star(x_i, a)}^2,
    \]
    and taking expectation over $(x_i, s_i)$,
    \[
        \EE\sbr{Z_i(f)}
        =
        \EE_{x \sim \Dcal_\Xcal}\sbr{
            \EE_{s \sim \pi(\cdot \mid x)}\sbr{
                \sum_{a: s(a) = 1} \del{f(x, a) - f^\star(x, a)}^2
            }
        }
        =
        \EE_{x \sim \Dcal_\Xcal}\sbr{ \opnorm{f - f^\star}{\bar\pi(x)}^2 },
    \]
    where the last equality uses $\bar\pi(x)(a) = \EE_{s \sim \pi(\cdot \mid x)}\sbr{s(a)}$.

    Since $f, f^\star, o_i(a) \in [0, 1]$ and $\norm{s_i}_1 \leq m$, we have $\abs{Z_i(f)} \leq m$. Factoring each term as a difference of squares,
    \[
        Z_i(f)
        =
        \sum_{a: s_i(a) = 1}
        \del{f(x_i, a) - f^\star(x_i, a)}
        \del{f(x_i, a) + f^\star(x_i, a) - 2 o_i(a)},
    \]
    so by the Cauchy--Schwarz inequality,
    \begin{align*}
        Z_i(f)^2
        &\leq
        \del{\sum_{a: s_i(a) = 1} \del{f(x_i, a) - f^\star(x_i, a)}^2}
        \del{\sum_{a: s_i(a) = 1} \del{f(x_i, a) + f^\star(x_i, a) - 2 o_i(a)}^2} \\
        &\leq
        4 m \cdot \EE\sbr{Z_i(f) \mid x_i, s_i},
    \end{align*}
    where we used that each summand in the second factor is at most $4$ and there are at most $m$ summands. Taking expectation gives $\EE\sbr{Z_i(f)^2} \leq 4 m\, \EE\sbr{Z_i(f)}$.

    Applying Bernstein's inequality and a union bound over $f \in \Fcal$, with probability at least $1 - \delta$, simultaneously for all $f \in \Fcal$,
    \[
        \EE\sbr{Z_1(f)} - \frac{1}{n} \sum_{i=1}^n Z_i(f)
        \lesssim
        \sqrt{\frac{m\, \EE\sbr{Z_i(f)} \log(\abs{\Fcal}/\delta)}{n}}
        +
        \frac{m}{n} \log\del{\frac{\abs{\Fcal}}{\delta}}.
    \]
    Substituting $f = \hat f$ and using $\frac{1}{n} \sum_{i=1}^n Z_i(\hat f) \leq 0$ (by the ERM property), the elementary implication $x \leq \sqrt{cxA} + B \Rightarrow x \lesssim cA + B$ for $A, B \geq 0$ yields
    \[
        \EE\sbr{Z_1(\hat f)}
        \lesssim
        \frac{m}{n} \log\del{\frac{\abs{\Fcal}}{\delta}},
    \]
    which is the desired bound by the expression for $\EE\sbr{Z_1(\hat f)}$ derived above.

\end{proof}